\documentclass[10pt,twocolumn,letterpaper]{article}

\usepackage{cvpr}              % To produce the CAMERA-READY version
\usepackage{algorithm}
\usepackage{algorithmic}
\renewcommand{\algorithmicrequire}{\textbf{Input:}}
\renewcommand{\algorithmicensure}{\textbf{Output:}}
\newcommand{\INPUT}{\item[\algorithmicrequire]}
\newcommand{\OUTPUT}{\item[\algorithmicensure]}

\usepackage{amsmath}
\usepackage{amssymb}
\usepackage{amsfonts}

\usepackage{booktabs}
\usepackage{multirow}
\usepackage{subcaption}

\definecolor{light-gray}{gray}{0.9}
\usepackage{colortbl}

\newcommand{\bheading}[1]{{\vspace{0.1\baselineskip}\noindent{\textbf{#1}}}}
\newcommand{\lheading}[1]{{\vspace{0.1\baselineskip}\noindent{\textit{#1}}}}

\usepackage{amsthm}
\theoremstyle{definition}
\newtheorem{definition}{Definition}[section]
\newtheorem{theorem}{Theorem}[section]
\newtheorem{lemma}{Lemma}[section]

\newtheorem{remark}{Remark}[section]

\definecolor{cvprblue}{rgb}{0.21,0.49,0.74}
\usepackage[pagebackref,breaklinks,colorlinks,allcolors=cvprblue]{hyperref}

\def\paperID{1603} % *** Enter the Paper ID here
\def\confName{CVPR}
\def\confYear{2025}

\title{Few-for-Many Personalized Federated Learning}

\author{Ping Guo$^{1,6}$, Tiantian Zhang$^{2}\thanks{Corresponding author: {\tt\small ttzhang@hkmu.edu.hk}}$, Xi Lin$^{3}$, Xiang Li$^{4}$, Zhi-Ri Tang$^{5}$, Qingfu Zhang$^{1,6}\thanks{Corresponding author: {\tt\small qingfu.zhang@cityu.edu.hk}}$\\
 $^1$City University of Hong Kong; \ \ $^2$Hong Kong Metropolitan University; \\ $^3$Xi'an Jiaotong University; $^4$Southeast University; $^5$Jinan University; \\ $^6$CityU Shenzhen Research Institute\\
}

\begin{document}
\maketitle
\begin{abstract}

    Personalized Federated Learning (PFL) aims to train customized models for clients with highly heterogeneous data distributions while preserving data privacy. Existing approaches often rely on heuristics like clustering or model interpolation, which lack principled mechanisms for balancing heterogeneous client objectives. Serving $M$ clients with distinct data distributions is inherently a multi-objective optimization problem, where achieving optimal personalization ideally requires $M$ distinct models on the Pareto front. However, maintaining $M$ separate models poses significant scalability challenges in federated settings with hundreds or thousands of clients. To address this challenge, we reformulate PFL as a few-for-many optimization problem that maintains only $K$ shared server models ($K \ll M$) to collectively serve all $M$ clients. We prove that this framework achieves near-optimal personalization: the approximation error diminishes as $K$ increases and each client's model converges to each client's optimum as data grows. Building on this reformulation, we propose FedFew, a practical algorithm that jointly optimizes the $K$ server models through efficient gradient-based updates. Unlike clustering-based approaches that require manual client partitioning or interpolation-based methods that demand careful hyperparameter tuning, FedFew automatically discovers the optimal model diversity through its optimization process. Experiments across vision, NLP, and real-world medical imaging datasets demonstrate that FedFew, with just 3 models, consistently outperforms other state-of-the-art approaches. Code is available at \url{https://github.com/pgg3/FedFew}.

\end{abstract}

\section{Introduction} \label{sec:intro}

Personalized Federated Learning (PFL)~\cite{smith2017federated} aims to train client-specific models that best fit each client's local data distribution by leveraging aggregated knowledge from collaborative learning across the federation. This paradigm overcomes a fundamental limitation of traditional Federated Learning~(FL)~\cite{mcmahan2017communication}, where a single global model struggles to effectively serve all clients when their data are drawn from vastly different distributions $P_i$ under the non-IID setting. The benefit of personalized collaborative learning has made PFL particularly valuable in domains such as healthcare~\cite{lu2022personalized} and finance~\cite{chatterjee2023federated}, where heterogeneous data distributions necessitate client-specific models while privacy constraints require decentralized training.

The heterogeneity of client data changes the nature of the optimization landscape in PFL~\cite{li2020federated,wang2020tackling}.
When clients have distinct data distributions $P_i \neq P_j$, a model update that benefits one client may harm another due to their inherent conflicting objectives~\cite{zhao2018federated}. For instance, in a healthcare federation where hospitals in urban and rural areas have vastly different patient demographics, optimizing model accuracy for urban hospitals might learn feature representations that poorly capture rural patient characteristics~\cite{liu2021feddg,guo2021multi}.

This inherent conflict naturally leads to a multi-objective optimization perspective~\cite{smith2017federated,hu2022federated}.
Rather than seeking a single consensus model, PFL must navigate trade-offs among $M$ distinct client loss functions $\{L_1, L_2, \ldots, L_M\}$, where each client $i$ requires a model tailored to its local data distribution $P_i$. While achieving optimal personalization would ideally require learning $M$ distinct models, maintaining and training such a large number of separate models introduces prohibitive scalability challenges in federated settings involving hundreds or thousands of clients.

\begin{figure*}[t]
    \centering
    \includegraphics[width=0.98\linewidth]{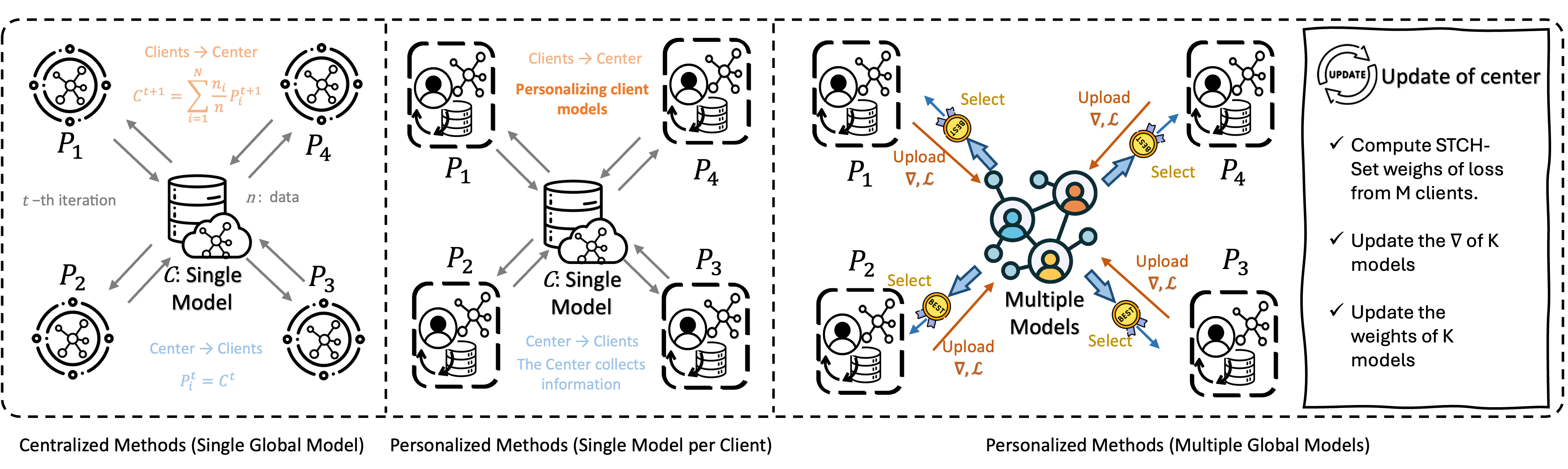}
    \vspace{-0.75\baselineskip}
    \caption{\textbf{Paradigms of Personalized Federated Learning.} \textbf{Left:} Centralized methods maintain a single global model for all $M$ clients, failing to capture client heterogeneity. \textbf{Center:} Per-client methods train $M$ independent models, sacrificing collaborative learning benefits and suffering from data scarcity. \textbf{Right:} Our proposed few-for-many approach maintains $K$ server models ($K \ll M$) that collectively serve all clients. Each client selects the best-fitting model, achieving strong personalization while preserving collaboration. }
    \label{fig:paradigms}
    \vspace{-1.5\baselineskip}
\end{figure*}

Given this scalability challenge, existing PFL methods struggle to effectively balance personalization and computational efficiency.
Methods that explicitly adopt the multi-objective perspective, such as FedMGDA~\cite{hu2022federated} and FedMTL~\cite{smith2017federated}, only obtain a single model on the \emph{Pareto front}~\cite{miettinen1999nonlinear}, which is the set of all optimal trade-off solutions. Consequently, they cannot provide individual optima for each client. Meanwhile, most PFL methods resort to heuristics without theoretical Pareto optimality guarantees: clustering-based methods like IFCA~\cite{ghosh2020efficient} and CFL~\cite{sattler2020clustered} train one model per group; interpolation-based methods like APFL~\cite{deng2020adaptive} and Ditto~\cite{tian2021ditto} mix global and local models using ad-hoc weights; and per-client methods like FedRep~\cite{collins2021exploiting} train $M$ independent models, thereby sacrificing collaboration benefits.
In summary, existing approaches face a significant limitation:
\emph{multi-objective methods produce only a single Pareto-optimal model without personalization, while heuristic methods generate $M$ personalized models without Pareto optimality guarantee.}

To address this challenge, we reformulate PFL as a few-for-many optimization problem that maintains only $K$ server models that collectively serve all $M$ clients (where $K \ll M$) as illustrated in Figure~\ref{fig:paradigms}, where each client selects the model that best fits its local data distribution. We rigorously prove that the $K$-for-$M$ framework achieves near-optimal personalization through a precise error decomposition. Our analysis establishes two vanishing error components: the \emph{Pareto coverage gap} from using $K < M$ models diminishes as $K$ increases, and the \emph{statistical error} between empirical and population losses vanishes as client dataset sizes grow. These dual convergence guarantees distinguish our approach from existing heuristic PFL methods.

Building upon this theoretical foundation, we propose \textbf{FedFew} (Federated Learning with Few Models), a novel algorithm that enables efficient gradient-based optimization in federated settings. The core challenge lies in jointly optimizing the $K$ server models alongside discrete client-model assignments, which is incompatible with standard gradient-based methods. We address this through a two-level smoothing technique that transforms the discrete selection problem into a fully differentiable objective. This formulation enables clients to perform soft model selection via gradient descent, while the server jointly updates all $K$ models to collaboratively cover all client needs. Unlike clustering-based approaches that require manual client partitioning or interpolation-based methods that demand careful hyperparameter tuning, FedFew can automatically discover the optimal model diversity through its optimization process. The algorithm seamlessly integrates with standard federated learning protocols, incurring only minimal communication overhead compared to single-model training.

Our contributions are three-fold:
\begin{itemize}[leftmargin=*,topsep=2pt,itemsep=2pt]
    \item We introduce the few-for-many optimization framework that reformulates PFL as maintaining $K$ shared models ($K \ll M$) to serve $M$ clients, addressing the scalability challenge with rigorous convergence guarantees through Pareto coverage gap and statistical error decomposition.

    \item We develop FedFew, a practical federated algorithm that solves the few-for-many problem via two-level smoothing, enabling automatic model diversity discovery through gradient-based optimization without manual client clustering or delicate hyperparameter tuning.

    \item We demonstrate through extensive experiments on seven datasets, including real-world medical imaging application, that FedFew consistently outperforms existing methods while utilizing only 3 models.
\end{itemize}

\section{Related Work}\label{sec:related}
\subsection{Standard and Personalized Federated Learning}

\bheading{Standard Federated Learning.}
Traditional approaches in federated learning aim to learn a single global model by aggregating local updates from distributed clients.
Starting with FedAvg~\cite{mcmahan2017communication}, which introduced weighted averaging, subsequent methods like FedProx~\cite{li2020federated}, SCAFFOLD~\cite{karimireddy2020scaffold}, and FedDyn~\cite{acar2021federated} have brought improvements in training convergence and stability.

While FL methods perform well under the IID data assumption, real-world FL problems often exhibit significant data heterogeneity across clients.
This mismatch can lead to degraded performance and slow convergence~\cite{zhao2018federated,li2021survey}.
Therefore, personalized federated learning approaches have been proposed to address this challenge by tailoring models to individual client distributions while still leveraging collaborative learning benefits.

\bheading{Personalized Federated Learning.}
Existing PFL methods can be roughly grouped into three categories: centralized methods with a single global model, personalized methods with one model per client, and personalized methods with multiple server models.

\lheading{Centralized Methods (Single Global Model).}
While the standard FL methods mentioned above (FedAvg~\cite{mcmahan2017communication}, FedProx~\cite{li2020federated}, etc.) maintain a single global model, they serve as important baselines for evaluating personalized approaches. Their limitation in handling heterogeneous data motivates the development of personalized methods.

\lheading{Personalized Methods (Single Model per Client).}
These methods maintain a unique personalized model for each client while leveraging knowledge from other clients.
Representation-based approaches like FedRep~\cite{collins2021exploiting} and FedBABU~\cite{oh2022fedbabu} decouple the model into shared and personalized components. Bi-level optimization methods such as Ditto~\cite{tian2021ditto}, pFedMe~\cite{dinh2020personalized}, and Per-FedAvg~\cite{fallah2020personalized} formulate personalization as a nested optimization problem.
Model interpolation approaches blend global and local models to achieve personalization. For example, FedBN~\cite{li2021fedbn} personalizes only batch normalization layers, while APFL~\cite{deng2020adaptive} learns explicit weights to mix global and local models. More recent methods like FedFomo~\cite{zhang2021personalized} and FedAMP~\cite{huang2021personalized} investigate adaptive mixing strategies.

\lheading{Personalized Methods (Multiple Server Models).}
These methods learn a small set of specialized models on the server by grouping clients with similar data distributions. IFCA~\cite{ghosh2020efficient} and CFL~\cite{sattler2020clustered} employ explicit hard clustering algorithms to assign each client to one model cluster. More recent approaches like FedSoft~\cite{ruan2022fedsoft}, FeSEM~\cite{long2023multi}, and PACFL~\cite{vahidian2023efficient} leverage soft clustering or expectation-maximization techniques for more flexible client-model associations. However, these methods lack theoretical guarantees on the quality of the learned model set and rely on heuristic clustering objectives.

\subsection{Multi-Objective Optimization in FL}

\noindent\textbf{Multi-Objective Optimization.}
Classical multi-objective optimization (MOO) approaches, such as weighted sum, Tchebycheff scalarization, and Normal Boundary Intersection~\cite{zhang2007moea}, aim to identify a set of Pareto-optimal solutions with various trade-offs among objectives. Recent gradient-based methods like MGDA~\cite{desideri2012multiple}, PCGrad~\cite{yu2020gradient}, and ParetoMTL~\cite{lin2019pareto} address multi-objective optimization by balancing conflicting gradients across objectives during the optimization process. Most recently, set scalarization methods~\cite{lin2025few} have emerged, which propose to approximate the Pareto front with a small solution set.

% \xl{I am not sure why set scalarization can "approximate the complete Pareto front with a small solution set under $\varepsilon$-Pareto optimality guarantees."}
% \TODO{Revise the last sentence (after revising all the following sections)}

\noindent\textbf{MOO in Federated Learning.}
Several PFL approaches have recognized the multi-objective nature of federated learning and attempted to address it through multi-task learning frameworks. FedMTL~\cite{smith2017federated} leverages task relationship matrices to enable knowledge transfer between clients. More recently, FedMGDA~\cite{hu2022federated} adopts Multiple Gradient Descent Algorithm (MGDA) to balance conflicting client objectives by finding a common gradient direction that benefits all clients.
While theoretically principled, these approaches involve complex bi-level optimization procedures and incur significant communication overhead, making them computationally prohibitive for large-scale federated settings.
Moreover, these methods target a single trade-off solution, limiting their ability to handle heterogeneous client preferences.

% \TODO{Clearly introduce the concept like Pareto optimality and $\varepsilon$-Pareto guarantees (in the introduction?) in advance.}

\section{PFL as Multi-Objective Optimization}
\label{sec:pfl_mop}

\subsection{Problem Setup and Client Objectives}

Consider a federated learning system with $M$ clients, where each client $i$ possesses a local dataset $D_i$ drawn from a distinct distribution $P_i$. The goal of each client is to find a model that minimizes its expected loss:
\begin{equation}
    \begin{aligned}
                           & \theta_i^* = \arg\min_{\theta} L_i(\theta),                     \\
        \text{where} \quad & L_i(\theta) = \mathbb{E}_{(x,y) \sim P_i}[\ell(f(x;\theta), y)]
    \end{aligned}
\end{equation}
where $\ell: \mathbb{R}^d \times \mathcal{Y} \to \mathbb{R}$ is the loss function and $f(\cdot;\theta)$ denotes the model parameterized by $\theta$.

In practice, clients work with empirical risk minimization over their finite local datasets:
\begin{equation}
    \hat{L}_i(\theta) = \frac{1}{|D_i|} \sum_{(x,y) \in D_i} \ell(f(x;\theta), y).
\end{equation}

This setting reveals two key challenges of the PFL problem: \emph{Collaboration} and \emph{Multi-Objective Trade-offs}.

\bheading{Collaboration.}
Independent local training often leads to severe overfitting due to limited data availability at each client, with generalization error scaling as $\mathcal{O}(1/\sqrt{|\mathcal{D}_i|})$~\cite{shalev2014understanding}. This limitation motivates collaborative learning that leverages data from other clients.
% \xl{Might add an equation for FL that leverage data from all clients. (also see comments latter in this subsection).}

\bheading{Multi-Objective Trade-Offs.}
However, collaboration introduces a challenge: when client data distributions are heterogeneous (i.e., $P_i \neq P_j, \forall i \neq j$), optimizing for one client may degrade performance on others. This inherent conflict reveals that PFL is intrinsically an $M$-objective optimization problem:
\begin{equation}
    \label{eq:multi_objective}
    \min_{\theta} \mathbf{L}(\theta) = [L_1(\theta), L_2(\theta), \ldots, L_M(\theta)]^T
\end{equation}
where no single model $\theta$ can simultaneously minimize all objectives.
To characterize optimal solutions, we introduce the concept of Pareto optimality:
\begin{definition}[Pareto Optimality~\cite{miettinen1999nonlinear}]
    A model $\theta^*$ is \emph{Pareto optimal} if there exists no other model $\theta$ such that $L_i(\theta) \leq L_i(\theta^*)$ for all $i \in [M]$ with strict inequality for at least one client. The \emph{Pareto set} contains all Pareto optimal models, and the \emph{Pareto front} is the set of objective vectors $\{[L_1(\theta^*), \ldots, L_M(\theta^*)]^T : \theta^* \text{ is Pareto optimal}\}$.
\end{definition}

While the Pareto front contains all optimal trade-off models, directly approximating it with high precision becomes computationally prohibitive.
Specifically, achieving $\varepsilon$-accuracy, where every Pareto optimal point has its representative within $\varepsilon$ distance, requires $\mathcal{O}((1/\varepsilon)^{M-1})$ models~\cite{das1998normal}.
This exponential dependence on $M$ renders direct approximation infeasible: even with $M=10$ clients, achieving $\varepsilon=0.01$ requires $10^{18}$ models, which is far beyond any practical system's capacity.
% \xl{A concise definition of $\varepsilon$-accuracy is needed.}
% \xl{Should we add an illustration figure for multi-objective optimization and Pareto front? In addition, we need to find a way to emphasize that the number of required models grows exponentially.}

\bheading{A Key Insight.}
While the Pareto front is continuous and requires exponential models to fully approximate, achieving optimal personalization does not require approximating the entire front.
Instead, we only need to find $M$ distinct models on the Pareto front, with one tailored for each client's distribution.
However, maintaining $M$ separate models remains impractical: the communication and computation costs grow linearly with $M$, becoming prohibitive when serving hundreds or thousands of clients.

This motivates our practical $K$-for-$M$ framework: we maintain only $K$ models (where $K \ll M$) that collectively serve all clients. Each client selects the best-fitting model from this set, achieving effective personalization with tractable overhead.

% \xl{The logic flow of subsection 3.1 is smooth and clear. However, it seems that the idea/setting of FL/PFL has not been properly emphasized. For example, all equations in this subsection are not explicitly for FL/PFL. }

% \TODO{Revision is needed to highlight the idea key: **FL/PFL** as multi-objective optimization. Maybe adding one key equation of FL/PFL with a breif discussion/analysis could be good enough.}

% \xl{I know it could be difficult, but we might need a good argument to support the $K \ll M$ setting (rather than $K = M$).}

\subsection{Set-based Optimization: K-for-M Framework}\label{subsec:set_k_4_m}

% Building on this insight, we formalize a set-based optimization approach where the server maintains exactly $K$ models to serve all $M$ clients.

\noindent\textbf{K-for-M Reformulation.}
Let $\Theta = \{\theta_1, \ldots, \theta_K\}$ denote the set of $K$ models maintained by the server. Each client $i$ will be served by the model $\theta_{k_i}$ that minimizes its local loss $L_i(\theta_{k_i})$. This transforms the original multi-objective problem~(\ref{eq:multi_objective}) into\footnote{Several existing PFL methods (\emph{e.g.,} IFCA~\cite{ghosh2020efficient}) implicitly tackle this same K-for-M optimization problem, though with different solution approaches.}:
\begin{equation}
    \label{eq:k4m_formulation}
    \min_{\Theta} \mathbf{F}(\Theta) = \begin{bmatrix}
        \displaystyle\min_{k_1 \in \{1,\ldots,K\}} L_1(\theta_{k_1}) \\[0.5em]
        \vdots                                                       \\[0.5em]
        \displaystyle\min_{k_M \in \{1,\ldots,K\}} L_M(\theta_{k_M})
    \end{bmatrix}.
\end{equation}

% This reformulation maintains computational tractability while preserving personalization benefits. 
The framework provides a natural mechanism for quality control: by adjusting $K$, system designers can systematically trade off between personalization quality and computational cost.

\noindent\textbf{Impact of $K$.}
The choice of $K$ determines the trade-off between personalization capacity and system efficiency:
\begin{itemize}[leftmargin=*,topsep=2pt,itemsep=2pt]
    \item $K=1$: Degenerates to a single model, where global model training (\emph{e.g.,} FedAvg) can find one Pareto optimal solution but fails to provide personalization;
    \item $K=M$: Each client could potentially have its own personalized model;
    \item $1 < K < M$: Our operating regime, balancing personalization quality with communication efficiency.
\end{itemize}

\bheading{Convergence Analysis.}
% We now give the sample complexity for the empirical K-for-M solution to converge to the population optimum of the multi-objective problem~\eqref{eq:multi_objective}. 
The following theorem characterizes the convergence rate in terms of two error components: the Pareto coverage gap and the statistical error.

\begin{theorem}[Convergence of K-for-M Framework]
    \label{thm:convergence}
    Let $\Theta^{(K)} = \{\theta_1, \ldots, \theta_K\}$ be the optimal solution with $K$ models for $M$ clients. Define $\Delta_{het} = \max_{i,j \in [M]} [L_i(\theta_j^*) - L_i(\theta_i^*)]$ as the maximum pairwise heterogeneity. Then the average error across clients is bounded by:
    \begin{align}
        \frac{1}{M}\sum_{i=1}^M \left\{\mathbb{E}\left[\min_{k \in [K]} L_i(\theta_k)\right] - L_i(\theta_i^*)\right\} \notag \\
        \qquad \leq \underbrace{\frac{M-K}{M} \cdot \Delta_{het}}_{\text{Pareto coverage gap}} + \underbrace{\mathcal{O}\left(\sqrt{\frac{Kd}{n}}\right)}_{\text{statistical error}}
    \end{align}
    where $\theta_i^* = \arg\min_\theta L_i(\theta)$ is client $i$'s optimal personalized model, $d$ is the model complexity, and $n$ is the average sample size per client. The complete proof is provided in the supplementary material.
\end{theorem}

\begin{remark}[Convergence to Optimal Solution]
    The bound decomposes the approximation error into two independent dimensions: (i) when $K=M$, the Pareto coverage gap vanishes, recovering individual personalized models for each client; (ii) as the local dataset size $n \to \infty$, the statistical error vanishes, ensuring the empirical solution converges to the population optimum. Achieving zero error requires both conditions simultaneously.
\end{remark}

% \xl{The term "two paths" might be not accurate, since we need **both** (i) and (ii) to achieve a 0 bound (global optimal solution?).}

\section{FedFew Algorithm}
\label{sec:algorithm}

\subsection{Smooth Tchebycheff Set Scalarization}

The K-for-M formulation~\eqref{eq:k4m_formulation} is an $M$-objective optimization problem, where each objective involves selecting the best model from a set $\Theta$.
To solve this problem while ensuring Pareto optimality, we adopt the Tchebycheff set scalarization (TCH-Set) approach, which transforms the multi-objective problem into a single scalar objective~\cite{zhang2007moea,lin2025few}:
\begin{equation}
    g^{\text{TCH-Set}}(\Theta|\lambda) = \max_{1 \leq i \leq M} \left\{\lambda_i\left(\min_{1 \leq k \leq K} L_i(\theta_k) - z_i^*\right)\right\}
    \label{eq:tch_set}
\end{equation}
where $\lambda = (\lambda_1, \ldots, \lambda_M)$ are client preference weights and $z_i^*$ is the ideal loss value for client $i$.

% \TODO{A brief discussion (motivation and benefit) of the Tchebycheff set scalarization, especially under the FL/PFL setting.}

This scalarization is particularly suited for personalized federated settings because: (i) it guarantees Pareto optimality of the solutions and (ii) it naturally handles heterogeneous client objectives without requiring explicit aggregation.
However, the nested $\max$ and $\min$ operators make~\eqref{eq:tch_set} non-differentiable, preventing gradient-based optimization required for federated training.

\bheading{Two-Level Smoothing.}
Since both $\max$ and $\min$ operators are non-differentiable, we employ log-sum-exp smoothing to enable gradient-based optimization~\cite{lin2025few,guo2025mos}:
\begin{align}
    \max_i \{x_i\} & \approx \mu \log\left(\sum_i \exp(x_i/\mu)\right) \label{eq:smooth_max}   \\
    \min_i \{x_i\} & \approx -\mu \log\left(\sum_i \exp(-x_i/\mu)\right) \label{eq:smooth_min}
\end{align}
where $\mu > 0$ controls the approximation quality.

\bheading{Final Formulation.}
For simplicity, we set $\lambda_i = 1$ and $z_i^* = 0$. Applying the two-level smoothing~\eqref{eq:smooth_max} and~\eqref{eq:smooth_min} to~\eqref{eq:tch_set}, we obtain the smooth Tchebycheff set scalarization (STCH-Set):
\begin{equation}
    g^{\text{STCH-Set}}(\Theta) = \mu \log \sum_{i=1}^M \left(\sum_{k=1}^K \exp\left(-\frac{L_i(\theta_k)}{\mu}\right)\right)^{-1}
    \label{eq:stch_set}
\end{equation}
where $\mu > 0$ is the smoothing parameter.\footnote{In implementation, we weight each client's loss $L_i(\theta_k)$ by its normalized sample size to account for varying local dataset sizes like FedAvg~\cite{mcmahan2017communication} before aggregation, i.e., $L_i(\theta_k) \gets \frac{n_i}{\sum_{j=1}^M n_j} \cdot L_i(\theta_k)$.}

\subsection{Decomposed Gradient Computation}

Taking the gradient of $g^{\text{STCH-Set}}$ in \eqref{eq:stch_set} with respect to $\theta_k$:
\begin{equation}
    \nabla_{\theta_k} g^{\text{STCH-Set}} = \sum_{i=1}^M \alpha_i \cdot w_{ik} \cdot \nabla_{\theta_k} L_i(\theta_k)
    \label{eq:gradient}
\end{equation}
where the weights decompose into two components. Define $S_i = \sum_{k=1}^K \exp(-L_i(\theta_k)/\mu)$. Then:
\begin{align}
     & \text{(Outer weight)} \quad \alpha_i = \frac{S_i^{-1}}{\sum_{j=1}^M S_j^{-1}}, \label{eq:alpha} \\
     & \text{(Inner weight)} \quad w_{ik}    = \frac{\exp(-L_i(\theta_k)/\mu)}{S_i}.  \label{eq:w}
\end{align}
The outer weight $\alpha_i$ assigns higher importance to clients with larger $S_i^{-1}$ (i.e., clients that perform poorly across all models), thereby implementing a hard-sample mining effect. The inner weight $w_{ik}$ performs soft model selection by assigning higher weights to models with lower loss for each client $i$, thus smoothly identifying the best-matching model for that client.

\subsection{Federated Implementation}

FedFew alternates between client gradient computation and server model updates through smooth Tchebycheff set scalarization. The federated optimization proceeds in communication rounds as outlined in Algorithm~\ref{alg:few_for_many}.

\bheading{Client Side:} Each client $i$ computes local gradients $g_{ik} = \nabla_{\theta_k} L_i(\theta_k)$ for all $K$ models and sends them to the server.

\bheading{Server Side:} The server computes weights $\{\alpha_i, w_{ik}\}$ from current losses $\{L_i(\theta_k)\}$ and aggregates:
\begin{equation}
    \nabla_{\theta_k} g^{\text{STCH-Set}} = \sum_{i=1}^M \alpha_i \cdot w_{ik} \cdot g_{ik}.
    \label{eq:server_agg}
\end{equation}

\begin{algorithm}[t]
    \caption{FedFew: Few-for-Many PFL with Smooth Tchebycheff Set Scalarization}
    \label{alg:few_for_many}
    \begin{algorithmic}[1]
        \INPUT $M$ clients with datasets $\{\mathcal{D}_i\}_{i=1}^M$, $K$ initial models $\Theta^{(0)}$, smoothing parameter $\mu$, learning rate $\eta$, local epochs $E$, communication rounds $T$
        \OUTPUT Optimized model set $\Theta^{(T)}$

        \FOR{$t = 1, 2, \ldots, T$}
        \STATE Server broadcasts $\Theta^{(t-1)}$ to all clients

        \FOR{each client $i = 1, 2, \ldots, M$ \textbf{in parallel}}
        \FOR{each model $k = 1, 2, \ldots, K$}
        \FOR{$e = 1, 2, \ldots, E$}
        \STATE Update local model: $\theta_k \leftarrow \theta_k - \eta \nabla_{\theta_k} L_i(\theta_k)$
        \ENDFOR
        \STATE Compute gradient $g_{ik}^{(t)} = \nabla_{\theta_k} L_i(\theta_k)$ and loss $L_i^{(t)}(\theta_k)$
        \ENDFOR
        \STATE Send $\{(g_{ik}^{(t)}, L_i^{(t)}(\theta_k))\}_{k=1}^K$ to server
        \ENDFOR

        \FOR{each model $k = 1, 2, \ldots, K$}
        \STATE Compute STCH-Set weights $\{\alpha_i, w_{ik}\}$ from losses $\{L_i^{(t)}(\theta_k)\}$ using Eqs.~\eqref{eq:alpha} and~\eqref{eq:w}
        \STATE $\nabla_{\theta_k} g^{\text{STCH-Set}} = \sum_{i=1}^M \alpha_i \cdot w_{ik} \cdot g_{ik}^{(t)}$
        \STATE $\theta_k^{(t)} = \theta_k^{(t-1)} - \eta \cdot \nabla_{\theta_k} g^{\text{STCH-Set}}$
        \ENDFOR
        \ENDFOR

        \STATE \textbf{Model Selection (post-training):}
        \STATE Server broadcasts trained models $\Theta^{(T)}$ to all clients
        \FOR{each client $i = 1, 2, \ldots, M$ \textbf{in parallel}}
        \STATE Client $i$ evaluates all $K$ models on local validation/training data
        \STATE Compute losses: $\{L_i(\theta_1^{(T)}), L_i(\theta_2^{(T)}), \ldots, L_i(\theta_K^{(T)})\}$
        \STATE Select best model: $k_i^* = \arg\min_{k \in \{1,\ldots,K\}} L_i(\theta_k^{(T)})$
        \ENDFOR
    \end{algorithmic}
\end{algorithm}

\bheading{Model Selection Mechanism.} After training, each client identifies the most suitable model from the $K$ available candidates through a simple local evaluation procedure. This process involves performing forward passes with all $K$ models on the client's local validation or training set, computing the corresponding losses, and selecting the model that achieves the minimum loss.
% \xl{Should we remove training set and only say validation set? 实验中确实用的是training set.}

\bheading{Communication Efficiency.}
% \xl{Should we remove the first sentence to avoid any potential concern?} 
In each communication round, every client performs $E$ local epochs of training and sends $K$ gradients along with $K$ scalar loss values to the server. The per-client communication cost is $\mathcal{O}(Kd)$ where $d$ is the model dimension. Since $K$ is typically small (ranging from 3 to 10 in practice) and remains fixed regardless of the total number of clients $M$, the resulting communication overhead factor of $K$ is modest. More importantly, by using local epochs $E > 1$, the number of required communication rounds $T$ can be reduced proportionally (See \S~\ref{subsec:communication}).

% , thereby achieving the same overall communication cost as FedAvg for comparable convergence. This characteristic makes our approach practical for large-scale federated learning deployments.

\subsection{Convergence Guarantees}\label{subsec:alg_convergence}

% The smooth approximation~\eqref{eq:stch_set} enables effective gradient-based optimization while preserving theoretical convergence guarantees. 
We establish two key theoretical properties:  uniform approximation quality and Pareto optimality guarantees.

\begin{theorem}[Uniform Smooth Approximation~\cite{lin2025few}]
    \label{thm:smooth_error}
    The smooth Tchebycheff set scalarization $g^{\text{STCH-Set}}(\Theta)$ uniformly approximates the non-smooth version $g^{\text{TCH-Set}}(\Theta)$. As the smoothing parameter $\mu \to 0$:
    \begin{equation}
        \lim_{\mu \to 0} g^{\text{STCH-Set}}(\Theta) = g^{\text{TCH-Set}}(\Theta)
        \label{eq:smooth_approx}
    \end{equation}
    uniformly over all model sets $\Theta$, with approximation error bounded by $\mathcal{O}(\mu \log M + \mu \log K)$.
\end{theorem}

The smoothing parameter $\mu$ controls the degree of smoothness in the approximation: smaller $\mu$ yields a tighter approximation to the original min-max objective, but results in sharper gradients that may hinder optimization.

\begin{theorem}[Pareto Properties of STCH-Set~\cite{lin2025few}]
    \label{thm:pareto_properties}
    The smooth Tchebycheff set scalarization provides strong Pareto guarantees:
    \begin{enumerate}
        \item \textbf{Pareto Optimality:} All solutions in the optimal set $\Theta^*_K$ are weakly Pareto optimal. Moreover, they are Pareto optimal if either the optimal set is unique or all preference coefficients are positive.
        \item \textbf{Pareto Stationarity:} If gradient descent converges to $\hat{\Theta} = \{\hat{\theta}_1, \ldots, \hat{\theta}_K\}$ where $\nabla_{\hat{\theta}_k} g^{\text{STCH-Set}}(\hat{\Theta}) = 0$ for all $k$, then all solutions in $\hat{\Theta}$ are Pareto stationary for the original multi-objective problem~\eqref{eq:multi_objective}.
    \end{enumerate}
\end{theorem}

% This theorem guarantees that our K-for-M framework finds solutions on the Pareto frontier. 

Combined with standard SGD convergence analysis, gradient descent on the smooth objective drives the expected squared gradient norm to $\mathcal{O}(1/\sqrt{T})$ after $T$ iterations. The overall approximation quality is controlled by both the optimization error ($\sim 1/\sqrt{T}$) and the smoothing error ($\sim \mu \log M + \mu \log K$).
Detailed proofs are provided in the supplementary material.

\bheading{Comparison with Clustering.}
Interestingly, clustering-based methods like IFCA~\cite{ghosh2020efficient} attempt to solve the same K-for-M optimization problem~\eqref{eq:k4m_formulation}. However, its hard client-to-cluster assignment creates a non-convex, discontinuous optimization landscape that lacks convergence guarantees. In contrast, our smooth Tchebycheff formulation ensures convergence to Pareto stationary points (Theorem~\ref{thm:pareto_properties}), demonstrating that the optimization strategy is crucial for both theoretical guarantees and practical performance.

\section{Experiments}\label{sec:experiments}

\subsection{Experimental Setup}

% Auto-generated table by scripts/generate_overall_table.py
% DO NOT EDIT MANUALLY
% Requires: \usepackage{booktabs} \usepackage{multirow} \usepackage{colortbl}
% Requires: \definecolor{light-gray}{gray}{0.9}

\begin{table*}[t]
    \centering
    \caption{Overall performance comparison across CIFAR-10, CIFAR-100, TINY (TinyImageNet), AG News, and FEMNIST datasets under pathological and practical heterogeneous settings. Results are reported as mean accuracy (\%) $\pm$ standard deviation. Best results are \colorbox{light-gray}{\textbf{highlighted}} and \textbf{bolded}, second-best results are \underline{underlined}.}
    \vspace{-0.75\baselineskip}
    \label{tab:overall_results}
    \resizebox{\textwidth}{!}{%
        \begin{tabular}{l c c c c c c c c c}
            \toprule
            \textbf{Method}                        & \multicolumn{2}{c}{\textbf{Pathological heterogeneous setting}}                     & \multicolumn{7}{c}{\textbf{Practical heterogeneous setting}}                                                                                                                                                                                                                                                                                                                                                                                                                                                                                                                                                                                                                                                   \\
            \cmidrule(lr){2-3} \cmidrule(lr){4-10}
                                                   & \multicolumn{2}{c}{\textbf{CIFAR-100}}                                              & \multicolumn{1}{c}{\textbf{FEMNIST}}                                                & \multicolumn{2}{c}{\textbf{CIFAR-10}}                                                & \multicolumn{2}{c}{\textbf{CIFAR-100}}                                              & \multicolumn{1}{c}{\textbf{TINY}}                                                   & \multicolumn{1}{c}{\textbf{AG News}}                                                                                                                                                                                                                                                                                                                  \\
            \cmidrule(lr){2-3} \cmidrule(lr){4-4} \cmidrule(lr){5-6} \cmidrule(lr){7-8} \cmidrule(lr){9-9} \cmidrule(lr){10-10}
                                                   & $M=10$                                                                              & $M=20$                                                                              & $M=20$                                                                               & $M=10$                                                                              & $M=20$                                                                              & $M=10$                                                                              & $M=20$                                                                              & $M=10$                                                                              & $M=20$                                                                              \\
            \midrule
            \multicolumn{10}{l}{\textit{Centralized Methods (Single Global Model)}}                                                                                                                                                                                                                                                                                                                                                                                                                                                                                                                                                                                                                                                                                                                                                                       \\
            FedAvg~\cite{mcmahan2017communication} & $\text{29.00} \pm \text{3.94}$                                                      & $\text{28.57} \pm \text{4.37}$                                                      & $\text{96.65} \pm \text{1.81}$                                                       & $\text{61.36} \pm \text{8.54}$                                                      & $\text{61.26} \pm \text{8.64}$                                                      & $\text{30.84} \pm \text{2.26}$                                                      & $\text{31.10} \pm \text{4.21}$                                                      & $\text{13.49} \pm \text{1.55}$                                                      & $\text{88.90} \pm \text{7.15}$                                                      \\
            FedProx~\cite{li2020federated}         & $\text{28.56} \pm \text{4.50}$                                                      & $\text{28.11} \pm \text{3.86}$                                                      & $\text{96.51} \pm \text{2.32}$                                                       & $\text{60.87} \pm \text{7.87}$                                                      & $\text{61.38} \pm \text{9.35}$                                                      & $\text{30.74} \pm \text{2.25}$                                                      & $\text{30.90} \pm \text{4.03}$                                                      & $\text{13.64} \pm \text{1.56}$                                                      & $\text{83.42} \pm \text{11.34}$                                                     \\
            FedMTL~\cite{smith2017federated}       & $\text{65.33} \pm \text{3.64}$                                                      & $\text{59.65} \pm \text{3.70}$                                                      & \cellcolor{light-gray}$\textbf{\text{100.00}} \boldsymbol{\pm} \textbf{\text{0.00}}$ & $\text{85.92} \pm \text{11.18}$                                                     & $\text{85.75} \pm \text{11.62}$                                                     & $\text{46.28} \pm \text{3.82}$                                                      & $\text{44.79} \pm \text{5.25}$                                                      & $\text{23.49} \pm \text{2.43}$                                                      & $\text{94.10} \pm \text{7.56}$                                                      \\
            \midrule
            \multicolumn{10}{l}{\textit{Personalized Methods (Single Model per Client)}}                                                                                                                                                                                                                                                                                                                                                                                                                                                                                                                                                                                                                                                                                                                                                                  \\
            APFL~\cite{deng2020adaptive}           & $\text{64.69} \pm \text{4.01}$                                                      & $\text{59.97} \pm \text{3.91}$                                                      & $\text{99.93} \pm \text{0.20}$                                                       & \cellcolor{light-gray}$\textbf{\text{88.38}} \boldsymbol{\pm} \textbf{\text{7.83}}$ & $\underline{\text{87.36} \pm \text{10.92}}$                                         & $\underline{\text{48.30} \pm \text{3.44}}$                                          & $\underline{\text{46.67} \pm \text{5.11}}$                                          & $\text{24.26} \pm \text{2.71}$                                                      & $\text{94.26} \pm \text{7.41}$                                                      \\
            Ditto~\cite{tian2021ditto}             & $\text{65.32} \pm \text{3.63}$                                                      & $\text{59.61} \pm \text{3.66}$                                                      & \cellcolor{light-gray}$\textbf{\text{100.00}} \boldsymbol{\pm} \textbf{\text{0.00}}$ & $\text{85.97} \pm \text{10.95}$                                                     & $\text{85.72} \pm \text{11.77}$                                                     & $\text{46.19} \pm \text{3.50}$                                                      & $\text{44.89} \pm \text{5.16}$                                                      & $\text{23.45} \pm \text{2.70}$                                                      & $\text{94.06} \pm \text{7.69}$                                                      \\
            FedRep~\cite{collins2021exploiting}    & \cellcolor{light-gray}$\textbf{\text{66.50}} \boldsymbol{\pm} \textbf{\text{3.41}}$ & $\underline{\text{61.46} \pm \text{3.82}}$                                          & \cellcolor{light-gray}$\textbf{\text{100.00}} \boldsymbol{\pm} \textbf{\text{0.00}}$ & $\text{87.36} \pm \text{8.92}$                                                      & $\text{86.94} \pm \text{10.76}$                                                     & $\text{48.26} \pm \text{3.31}$                                                      & $\text{46.46} \pm \text{4.52}$                                                      & $\underline{\text{27.24} \pm \text{2.77}}$                                          & $\underline{\text{94.68} \pm \text{12.69}}$                                         \\
            FedAMP~\cite{huang2021personalized}    & $\text{65.41} \pm \text{3.67}$                                                      & $\text{59.66} \pm \text{3.59}$                                                      & \cellcolor{light-gray}$\textbf{\text{100.00}} \boldsymbol{\pm} \textbf{\text{0.00}}$ & $\text{85.96} \pm \text{10.92}$                                                     & $\text{85.93} \pm \text{11.53}$                                                     & $\text{46.13} \pm \text{3.62}$                                                      & $\text{45.00} \pm \text{5.08}$                                                      & $\text{23.32} \pm \text{2.62}$                                                      & $\text{94.17} \pm \text{7.48}$                                                      \\
            \midrule
            \multicolumn{10}{l}{\textit{Personalized Methods (Multiple Server Models)}}                                                                                                                                                                                                                                                                                                                                                                                                                                                                                                                                                                                                                                                                                                                                                                   \\
            IFCA~\cite{ghosh2020efficient}         & $\text{42.34} \pm \text{5.18}$                                                      & $\text{43.89} \pm \text{3.58}$                                                      & $\text{99.46} \pm \text{0.78}$                                                       & $\text{77.27} \pm \text{7.14}$                                                      & $\text{73.35} \pm \text{12.00}$                                                     & $\text{34.09} \pm \text{5.69}$                                                      & $\text{29.80} \pm \text{3.75}$                                                      & $\text{15.24} \pm \text{1.90}$                                                      & $\text{90.63} \pm \text{11.79}$                                                     \\
            \textbf{FedFew (Ours)}                 & $\underline{\text{65.47} \pm \text{3.90}}$                                          & \cellcolor{light-gray}$\textbf{\text{64.98}} \boldsymbol{\pm} \textbf{\text{3.32}}$ & \cellcolor{light-gray}$\textbf{\text{100.00}} \boldsymbol{\pm} \textbf{\text{0.00}}$ & $\underline{\text{88.17} \pm \text{7.74}}$                                          & \cellcolor{light-gray}$\textbf{\text{88.26}} \boldsymbol{\pm} \textbf{\text{9.06}}$ & \cellcolor{light-gray}$\textbf{\text{50.44}} \boldsymbol{\pm} \textbf{\text{3.14}}$ & \cellcolor{light-gray}$\textbf{\text{53.69}} \boldsymbol{\pm} \textbf{\text{4.79}}$ & \cellcolor{light-gray}$\textbf{\text{30.31}} \boldsymbol{\pm} \textbf{\text{3.06}}$ & \cellcolor{light-gray}$\textbf{\text{96.07}} \boldsymbol{\pm} \textbf{\text{4.82}}$ \\
            \bottomrule
        \end{tabular}
    }
\end{table*}

\vspace{-0.5\baselineskip}

\begin{table*}[t]
    \vspace{-1\baselineskip}
    \begin{minipage}{0.60\textwidth}
        \captionof{table}{Performance comparison on medical imaging datasets (Kvasir and FedISIC). For each dataset, we report three metrics: (1) Avg: average accuracy across all clients with standard deviation; (2) Min: worst-case client accuracy; (3) Max: best-case client accuracy. Best results are \colorbox{light-gray}{\textbf{highlighted}} and \textbf{bolded}, second-best results are \underline{underlined}.}
        \vspace{-0.5\baselineskip}
        \label{tab:medical_results}
        \resizebox{\textwidth}{!}{
            % Auto-generated table by scripts/generate_medical_table.py
% DO NOT EDIT MANUALLY
% Requires: \usepackage{booktabs} \usepackage{multirow} \usepackage{colortbl}
% Requires: \definecolor{light-gray}{gray}{0.9}
% Note: This file only contains tabular content. Wrap it with table environment in the main document.

\centering
\footnotesize
\begin{tabular}{l c c c c c c}
\toprule
\textbf{Method} & \multicolumn{3}{c}{\textbf{Kvasir}} & \multicolumn{3}{c}{\textbf{FedISIC}} \\
\cmidrule(lr){2-4} \cmidrule(lr){5-7}
 & \textit{Avg. $\pm$ Std.} & \textit{Min.} & \textit{Max.} & \textit{Avg. $\pm$ Std.} & \textit{Min.} & \textit{Max.} \\
\midrule
\multicolumn{7}{l}{\textit{Local-only Baseline}} \\
Local-only & $\text{92.16} \pm \text{7.22}$ & $\text{80.49}$ & \cellcolor{light-gray}$\textbf{\text{100.00}}$ & $\text{65.37} \pm \text{17.23}$ & $\text{41.08}$ & $\text{94.54}$ \\
\midrule
\multicolumn{7}{l}{\textit{Centralized Methods (Single Global Model)}} \\
FedAvg~\cite{mcmahan2017communication} & $\text{85.96} \pm \text{2.30}$ & $\text{82.20}$ & $\text{89.20}$ & $\text{64.71} \pm \text{15.66}$ & $\text{48.50}$ & $\text{95.15}$ \\
FedProx~\cite{li2020federated} & $\text{79.91} \pm \text{11.24}$ & $\text{57.51}$ & $\text{86.58}$ & $\text{65.46} \pm \text{18.77}$ & $\text{42.93}$ & $\text{96.06}$ \\
FedMTL~\cite{smith2017federated} & $\text{92.46} \pm \text{6.75}$ & $\text{82.20}$ & \cellcolor{light-gray}$\textbf{\text{100.00}}$ & $\text{69.20} \pm \text{15.29}$ & $\underline{\text{54.19}}$ & $\text{96.46}$ \\
\midrule
\multicolumn{7}{l}{\textit{Personalized Methods (Single Model per Client)}} \\
APFL~\cite{deng2020adaptive} & $\text{91.97} \pm \text{7.00}$ & $\text{82.20}$ & $\text{99.77}$ & $\text{67.83} \pm \text{15.63}$ & $\text{52.92}$ & $\text{95.45}$ \\
Ditto~\cite{tian2021ditto} & $\text{92.37} \pm \text{7.17}$ & $\text{80.73}$ & \cellcolor{light-gray}$\textbf{\text{100.00}}$ & $\underline{\text{69.51} \pm \text{15.72}}$ & $\text{52.20}$ & $\underline{\text{96.66}}$ \\
FedRep~\cite{collins2021exploiting} & $\text{92.71} \pm \text{6.47}$ & $\underline{\text{82.93}}$ & \cellcolor{light-gray}$\textbf{\text{100.00}}$ & $\text{64.38} \pm \text{16.51}$ & $\text{47.96}$ & $\text{94.54}$ \\
FedAMP~\cite{huang2021personalized} & $\underline{\text{92.76} \pm \text{6.72}}$ & $\text{82.20}$ & \cellcolor{light-gray}$\textbf{\text{100.00}}$ & $\text{67.41} \pm \text{17.12}$ & $\text{45.84}$ & \cellcolor{light-gray}$\textbf{\text{96.76}}$ \\
\midrule
\multicolumn{7}{l}{\textit{Personalized Methods (Multiple Server Models)}} \\
IFCA~\cite{ghosh2020efficient} & $\text{82.05} \pm \text{21.88}$ & $\text{40.24}$ & \cellcolor{light-gray}$\textbf{\text{100.00}}$ & $\text{53.61} \pm \text{20.45}$ & $\text{23.23}$ & $\text{85.74}$ \\
\textbf{FedFew (Ours)} & \cellcolor{light-gray}$\textbf{\text{92.84}} \boldsymbol{\pm} \textbf{\text{6.08}}$ & \cellcolor{light-gray}$\textbf{\text{83.90}}$ & $\text{99.77}$ & \cellcolor{light-gray}$\textbf{\text{69.57}} \boldsymbol{\pm} \textbf{\text{14.59}}$ & \cellcolor{light-gray}$\textbf{\text{55.40}}$ & $\text{95.35}$ \\
\bottomrule
\end{tabular}
        }

    \end{minipage}
    \hfill
    \begin{minipage}{0.38\textwidth}
        \centering
        \begin{subfigure}{0.9\textwidth}
            \centering
            \includegraphics[width=\textwidth]{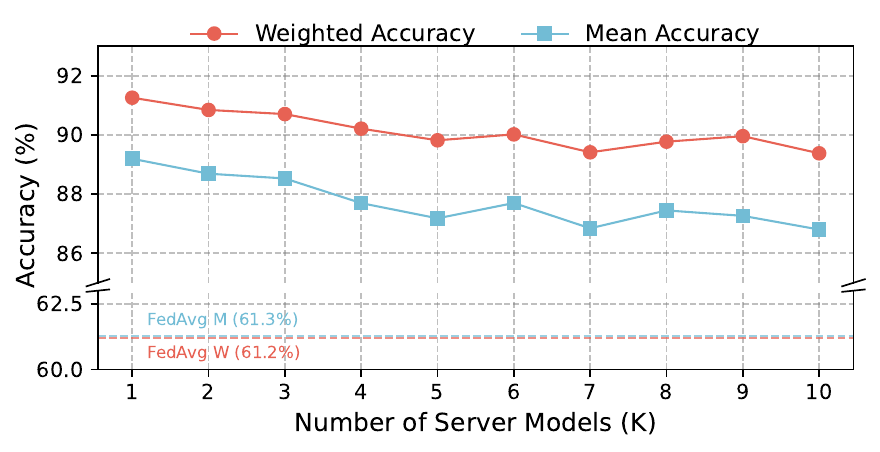}
            \caption{Impact of K}
            \label{fig:ablation_k}
        \end{subfigure}

        \begin{subfigure}{0.9\textwidth}
            \centering
            \includegraphics[width=\textwidth]{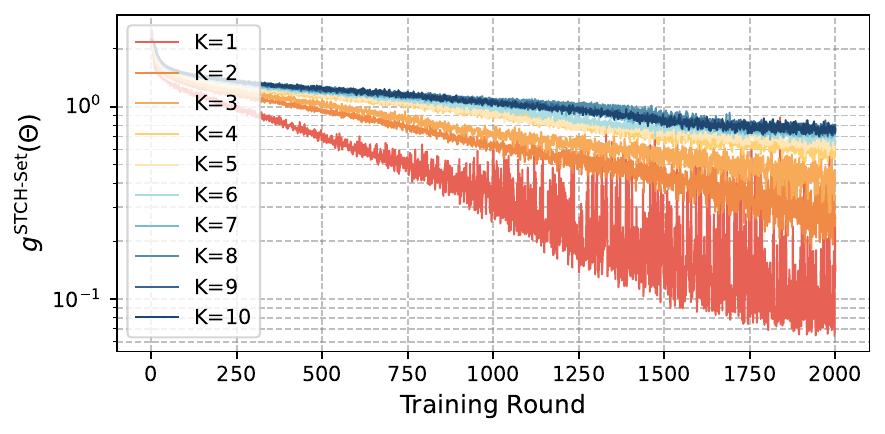}
            \caption{Convergence of $g^{\text{STCH-Set}}$}
            \label{fig:convergence}
        \end{subfigure}
        \vspace{-0.5\baselineskip}
        \captionof{figure}{\textbf{Sensitivity Studies.} (a) Test accuracy vs K on CIFAR-10. FedAvg baselines (dashed) shown for comparison. (b) Evolution over training rounds (log scale) for different K values.}
        \label{fig:ablation}
    \end{minipage}
    \vspace{-1.5\baselineskip}
\end{table*}

\bheading{Datasets.}
We evaluate our method on benchmark datasets with controlled heterogeneity and real-world medical imaging datasets under two distinct settings: \textit{pathological} (extreme label imbalance) and \textit{practical} (realistic label skew via Dirichlet distribution or natural partitions).

In benchmark datasets, under the \textit{pathological setting},
% For benchmark datasets, we employ standard vision and NLP datasets to evaluate performance under controlled heterogeneity. For the \textit{pathological setting}, 
we use CIFAR-100~\cite{krizhevsky2009learning} partitioned by assigning 2 classes per client ($M \in \{10, 20\}$), creating extreme label imbalance. Under the \textit{practical setting}, we partition data using Dirichlet ($\alpha=0.5$) distribution~\cite{zhao2018federated}, where smaller $\alpha$ induces stronger label skew.
Specifically, we evaluate on CIFAR-10~\cite{krizhevsky2009learning} and CIFAR-100 ($M \in \{10, 20\}$ clients), TinyImageNet ($M=10$), and AG News~\cite{zhang2015character} ($M=20$). We also include FEMNIST~\cite{caldas2018leaf} ($M=20$) with natural user-based partitioning.

We further validate our method on real-world medical datasets, where data heterogeneity arises from natural domain shifts across medical institutions. Kvasir~\cite{pogorelov2017kvasir} is a gastrointestinal disease detection dataset containing endoscopic images across 8 classes (polyps, ulcerative colitis, etc.), which we partition among $M=5$ clients using Dirichlet ($\alpha=0.5$) to simulate hospitals with different disease prevalence. FedISIC2019~\cite{NEURIPS2022_232eee8e} is a skin lesion classification dataset from the ISIC 2019 challenge with 8 diagnostic categories, where the data naturally originates from $M=6$ different medical centers, each with distinct imaging equipment and patient demographics, creating realistic cross-institutional heterogeneity.

% For real-world medical dataset, we include two challenging medical imaging datasets where data heterogeneity arises from natural domain shifts across medical institutions. Kvasir~\cite{pogorelov2017kvasir} is a gastrointestinal disease detection dataset containing endoscopic images across 8 classes (polyps, ulcerative colitis, etc.), which we partition among $M=5$ clients using Dirichlet($\alpha=0.5$) to simulate hospitals with different disease prevalence. FedISIC2019~\cite{NEURIPS2022_232eee8e} is a skin lesion classification dataset from the ISIC 2019 challenge with 8 diagnostic categories, where the data naturally originates from $M=6$ different medical centers, each with distinct imaging equipment and patient demographics, creating realistic cross-institutional heterogeneity.

\bheading{Baselines.}
We compare FedFew against nine baseline approaches spanning different personalization strategies. For \textit{centralized methods}, we include FedAvg~\cite{mcmahan2017communication} and FedProx~\cite{li2020federated}, which train a single global model shared by all clients. We also include FedMTL~\cite{smith2017federated}, a multi-task learning approach that learns task relationships to enable model personalization. For \textit{personalized methods with single model per client}, we evaluate APFL~\cite{deng2020adaptive}, Ditto~\cite{tian2021ditto}, FedRep~\cite{collins2021exploiting}, and FedAMP~\cite{huang2021personalized}, which maintain separate personalized models for each client. For \textit{personalized methods with multiple server models}, we compare against IFCA~\cite{ghosh2020efficient}, our core competitor that also uses $K$ shared models but relies on hard clustering. For medical datasets, we further include a \textit{local-only baseline} trained solely on local data without federation to assess the benefit of collaborative learning.

\bheading{Implementation.}
We employ a 4-layer CNN~\cite{mcmahan2017communication} for CIFAR-10/100, FEMNIST, and TinyImageNet, TextCNN~\cite{zhang2025pfllib} for AG News, and ResNet-18~\cite{he2016deep} for medical datasets. Our method utilizes $K=3$ server models across all experiments. Training proceeds for 2000 communication rounds on benchmark datasets and 1000 rounds on medical datasets to mitigate overfitting on smaller-scale medical data, with 1 local epoch per round. Batch sizes are configured based on dataset characteristics: 100 for datasets with lower overfitting tendency and 50 for those more susceptible to overfitting. Learning rates are selected according to dataset complexity, ranging from 0.0005 to 0.005. Full client participation is enforced in each communication round. Comprehensive hyperparameter configurations are in the supplementary material.

% We use the 4-layer CNN model~\cite{mcmahan2017communication} for CIFAR-10/100, FEMNIST, and TinyImageNet, TextCNN~\cite{zhang2025pfllib} for AG News, and ResNet-18~\cite{he2016deep} for the medical datasets (Kvasir and FedISIC). For our method, we set $K=3$ server models across all datasets. We train for 2000 communication rounds on benchmark datasets and 1000 rounds on medical datasets (to avoid overfitting on smaller medical datasets), with 1 local epoch per round. We select batch sizes based on overfitting tendency: batch size 100 for datasets less prone to overfitting (FEMNIST, AG News, Kvasir) and batch size 50 for others that are more susceptible to overfitting (CIFAR-10/100, TinyImageNet, FedISIC). Learning rates are tuned according to dataset complexity, ranging from 0.0005 to 0.01. All clients participate in each round. Detailed hyperparameters are provided in the supplementary material.

\subsection{Main Results}
\bheading{Benchmark Datasets.}
Table~\ref{tab:overall_results} presents the performance comparison on benchmark datasets under both pathological and practical heterogeneity settings.

Our FedFew method demonstrates superior performance across diverse data distributions and client configurations.
In pathological heterogeneous settings with CIFAR-100, FedFew achieves 64.98\% accuracy with $M=20$ clients, outperforming the best personalized baseline (FedRep at 61.46\%).
Under practical heterogeneous settings, FedFew consistently ranks first or second across all datasets.
Notably, on CIFAR-100 with $M=20$ clients, FedFew achieves 53.69\% accuracy, surpassing the best baseline by 7.02\%.
On TinyImageNet, FedFew improves over the strongest baseline (FedRep) by 3.07\%, demonstrating its effectiveness on large-scale image classification. For AG News text classification, FedFew achieves 96.07\% accuracy, outperforming FedRep by 1.39\%.

\bheading{Real-world Medical Dataset.}
Table~\ref{tab:medical_results} presents results on medical imaging datasets with naturally heterogeneous distributions.
On the Kvasir gastrointestinal dataset, FedFew achieves the highest average accuracy (92.84\%) and best worst-case performance, demonstrating robustness across diverse medical institutions.
For FedISIC skin lesion classification, FedFew attains 69.57\% average accuracy with 55.40\% minimum accuracy, significantly outperforming IFCA which suffers from severe performance degradation.

Notably, methods adopting multi-objective perspectives (FedFew and FedMTL) both achieve significantly higher minimum accuracies compared to other baselines (at least +1.2\% improvement over other baselines, +13.0\% over local-only on FedISIC), showcasing the advantage of multi-objective optimization in balancing performance across heterogeneous clients.

\begin{figure}[t]
    \centering
    \includegraphics[width=0.38\textwidth]{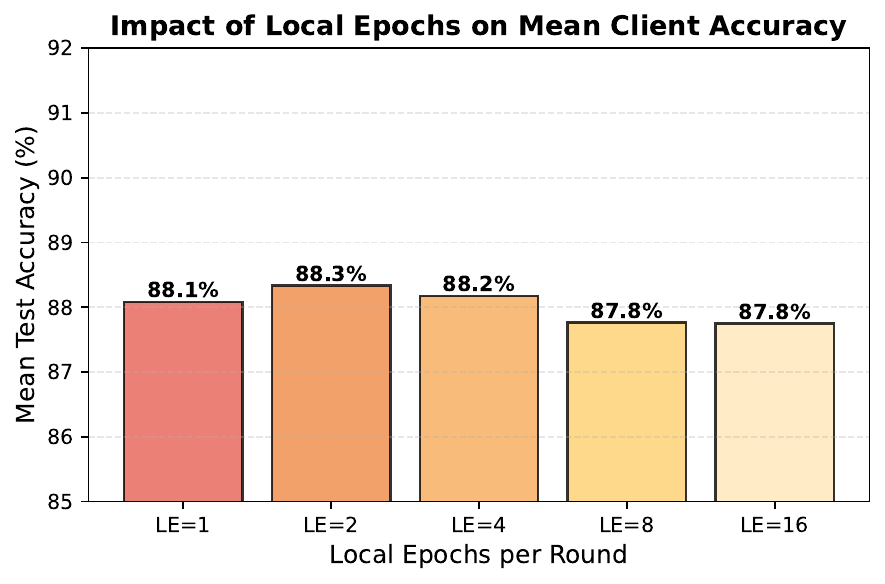}
    \vspace{-0.75\baselineskip}
    \caption{\textbf{Mean client accuracy comparison across communication configurations.} All configurations achieve comparable mean client accuracy (87.8--88.3\%), demonstrating that our method is robust to different communication-computation trade-offs. }
    \label{fig:rounds_accuracy_suppl}
    \vspace{-1.5\baselineskip}
\end{figure}

\subsection{Sensitivity Analysis and Convergence}

We conduct sensitivity analysis and convergence studies on CIFAR-10 with Dirichlet-$\alpha=0.5$ heterogeneity across $M=20$ clients. Throughout this section, we use $K=3$ server models by default, except when explicitly varying $K$ to study its impact on performance.

\subsubsection{Effect of Number of Server Models}~\label{subsec:ablation_K}
% To investigate the influence of the number of server models $K$ on performance and convergence, we vary $K$ from 1 to 10 and analyze the results.

\bheading{Robust Test Accuracy.} Figure~\ref{fig:ablation_k} presents test accuracy for $K \in \{1, 2, \ldots, 10\}$.
Our method achieves weighted accuracy ranging from 89.4 to 91.3\% across all K values, consistently demonstrating substantial improvement over FedAvg's 61.2\% baseline.
Notably, the single-model configuration ($K=1$) attains the highest accuracy of 91.3\%.
This non-monotonic relationship between $K$ and performance can be attributed to two factors: (1) \textit{Underlying data homogeneity:} CIFAR-10 is sampled from a single distribution, a single well-optimized model can perform well across clients; (2) \textit{Optimization complexity:} larger $K$ expands the parameter space, leading to slower convergence within the fixed training rounds.

\bheading{Convergence of STCH-Set Objective.}
To validate this optimization complexity hypothesis, we examine convergence behavior in Figure~\ref{fig:convergence}, which tracks the evolution of $g^{\text{STCH-Set}}(\Theta)$ over 2,000 training rounds (log scale).
We observe consistent monotonic decrease across all K values, confirming the stability of our gradient-based optimization.
However, with increasing $K$, the convergence speed decreases significantly, corroborating that larger model sets indeed create more challenging optimization landscapes that hinder both convergence rate and final performance.

% \textit{Correlation with accuracy.} Different K values converge to distinct final g values: K=1 achieves the lowest (g=0.14), while K=8 reaches the highest (g=0.81). This strongly correlates with test accuracy---configurations with lower final g values exhibit higher accuracy, confirming that minimizing our smoothed scalarization indeed drives toward better Pareto optimality. The smooth, oscillation-free convergence curves demonstrate the stability of gradient-based optimization even with non-convex neural network parameterization and stochastic client sampling.

\begin{figure}[t]
    \centering
    \includegraphics[width=0.38\textwidth]{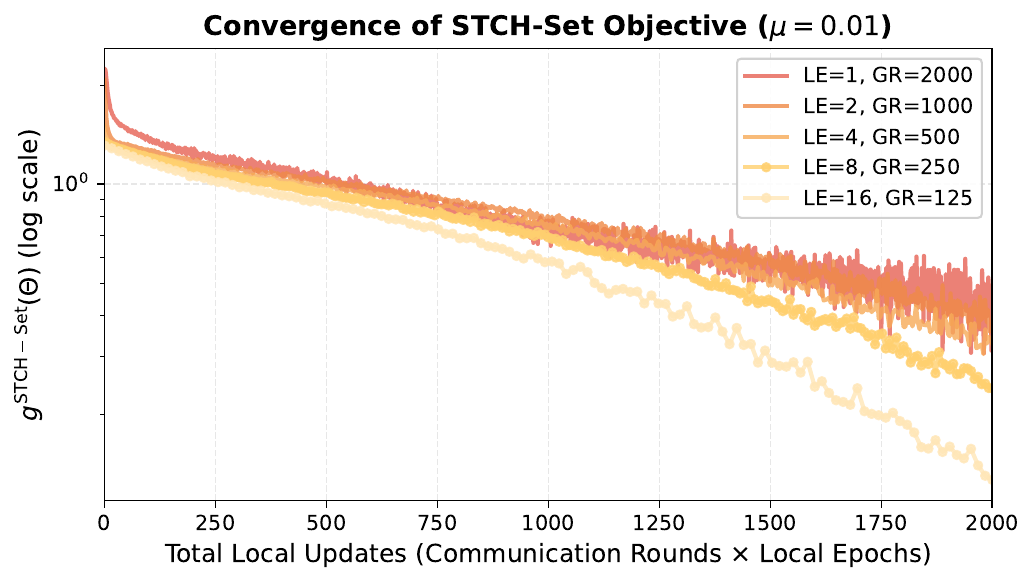}
    \vspace{-0.75\baselineskip}
    \caption{\textbf{Communication-computation trade-off.} Convergence of $g^{\text{STCH-Set}}$ vs total local updates for different (local epochs, communication rounds) configurations. Local epochs (LE) $\in \{1, 2, 4, 8, 16\}$ with corresponding communication rounds (GR) to maintain 2000 total updates.}
    \label{fig:ablation_rounds}
    \vspace{-1.5\baselineskip}
\end{figure}

\subsubsection{Communication Efficiency}\label{subsec:communication}
We examine the trade-off between communication frequency and local computation by varying the number of local epochs per round while maintaining constant total local updates (communication rounds $\times$ local epochs = 2000).

Figure~\ref{fig:ablation_rounds} shows the convergence of $g^{\text{STCH-Set}}(\Theta)$ across five configurations.
Configurations with more local epochs (LE=16) exhibit faster convergence and lower variance compared to frequent communication (LE=1). Specifically, LE=16 achieves the steepest descent and most stable optimization trajectory, demonstrating that our method maintains or even improves performance while drastically reducing communication overhead. All configurations reach comparable mean client accuracies as shown in Figure~\ref{fig:rounds_accuracy_suppl}.

\section{Conclusion}
\label{sec:conclusion}

In this paper, we propose FedFew, a novel personalized federated learning algorithm that tackles the scalability challenge in PFL through a Few-for-Many framework, where a small set of $K$ server models collaboratively serve $M$ clients with $K \ll M$. Our approach reformulates PFL as a multi-objective optimization problem and leverages the smooth Tchebycheff set scalarization for effective gradient optimization. Extensive experiments on multiple benchmarks, including healthcare collaborations and edge computing scenarios, demonstrate that FedFew achieves superior personalization performance while maintaining computational efficiency and scalability.

% The framework's ability to maintain personalization quality with significantly fewer models than clients opens new possibilities for practical deployment of personalized FL in resource-constrained environments.
%  such as healthcare collaborations and edge computing.

% ============================================================================

{
    \small
    \bibliographystyle{ieeenat_fullname}
    \bibliography{main}
}

% WARNING: do not forget to delete the supplementary pages from your submission 
\appendix
% Supplementary Material for "FedFew: Few-for-Many Personalized Federated Learning"

\section{Theoretical Analysis}
\label{sec:theory_suppl}

This supplementary material provides detailed proofs for the theorems presented in the main paper. We organize the material as follows: (A) proof of convergence of K-for-M framework (Theorem~\ref{thm:convergence} from Section~\ref{subsec:set_k_4_m}), (B) proof of uniform smooth approximation (Theorem~\ref{thm:smooth_error} from Section~\ref{subsec:alg_convergence}), and (C) proof of Pareto properties (Theorem~\ref{thm:pareto_properties} from Section~\ref{subsec:alg_convergence}), including both Pareto optimality and Pareto stationarity.

\subsection{Notation}
\label{sec:notation}

We begin by establishing the notation used for all proofs. Table~\ref{tab:notation} summarizes the key symbols and their definitions.

\subsection{Theorem~\ref{thm:convergence}: K-for-M Convergence}
\label{sec:proof_convergence}

We now provide the complete proof of Theorem~\ref{thm:convergence} from the main paper. Following the notation in the theorem statement, we use $\Theta^{(K)} = \{\theta_1, \ldots, \theta_K\}$ to denote the K-for-M solution obtained by minimizing empirical losses over finite samples.

\bheading{Proof roadmap.}
The proof proceeds in three steps. First, we establish that each model in the K-for-M solution lies on the Pareto frontier (Lemma~\ref{lem:k4m_pareto}). Second, we bound the Pareto coverage gap, which measures how well $K$ models approximate $M$ personalized optima (Lemma~\ref{lem:pareto_bound}). This bound depends on the maximum heterogeneity across clients. Third, we bound the statistical error arising from finite-sample learning (Lemma~\ref{lem:statistical}). Combining these two independent error sources yields the final convergence rate.

To quantify the approximation quality, we introduce the notion of maximum heterogeneity:

\begin{definition}[Maximum Heterogeneity]
    \label{def:heterogeneity}
    The maximum pairwise heterogeneity is defined as:
    \begin{equation}
        \Delta_{het} = \max_{i,j \in [M]} \left[L_i(\theta_j^*) - L_i(\theta_i^*)\right]
    \end{equation}
    This measures the worst-case loss degradation when a client uses another client's optimal model instead of its own.
\end{definition}

\begin{table}[t]
    \centering
    \caption{Key notation for theoretical analysis. Stars ($^*$) denote optimal or population-level quantities, while hats ($\hat{\cdot}$) denote empirical estimators.}
    \label{tab:notation}
    \begin{tabular}{lp{0.74\columnwidth}}
        \toprule
        \textbf{Symbol}     & \textbf{Description}                                                                \\
        \midrule
        $M$                 & number of heterogeneous clients                                                     \\
        \hline
        $K$                 & number of shared models in the K-for-M framework                                    \\
        \hline
        $L_i(\theta)$       & expected (population) loss for client $i$ with model $\theta$                       \\
        \hline
        $\hat{L}_i(\theta)$ & empirical loss for client $i$ based on finite samples                               \\
        \hline
        $\theta_i^*$        & optimal personalized model for client $i$, defined as $\arg\min_\theta L_i(\theta)$ \\
        \hline
        $\Theta^{(K)}_*$    & optimal K-for-M solution (minimizing population losses)                             \\
        \hline
        $\Theta^{(K)}$      & empirical K-for-M solution (minimizing empirical losses)                            \\
        \hline
        $n$                 & average sample size per client                                                      \\
        \hline
        $d$                 & VC dimension of hypothesis class $\Theta$                                           \\
        \bottomrule
    \end{tabular}
\end{table}

\subsubsection{Pareto Coverage Analysis}
\label{sec:pareto}
We analyze how well $K$ models can approximate the Pareto set. Following the Pareto optimality definition in the main paper, let $\mathcal{P} = \{\theta : \nexists \theta' \text{ s.t. } L_i(\theta') \leq L_i(\theta) \,\forall i \text{ with } L_j(\theta') < L_j(\theta) \text{ for some } j\}$ denote the set of all Pareto optimal models.

The K-for-M solution $\Theta^{(K)}_*$ is defined as:
\begin{equation}
    \Theta^{(K)}_* = \arg\min_{\theta_1, \ldots, \theta_K} \sum_{i=1}^M \min_{k \in [K]} L_i(\theta_k)
\end{equation}

\begin{lemma}[Pareto Optimality of K-for-M Solution]
    \label{lem:k4m_pareto}
    Each model $\theta_k^* \in \Theta^{(K)}_*$ is Pareto optimal, i.e., $\theta_k^* \in \mathcal{P}$.
\end{lemma}

\begin{proof}
    Suppose for contradiction that some $\theta_k^* \notin \mathcal{P}$. Then there exists a Pareto-dominating model $\theta' \in \mathcal{P}$ with $L_i(\theta') \leq L_i(\theta_k^*)$ for all $i$ and $L_j(\theta') < L_j(\theta_k^*)$ for some $j$. Replacing $\theta_k^*$ with $\theta'$ in $\Theta^{(K)}_*$ strictly decreases the objective, contradicting optimality.
\end{proof}

\bheading{Model Capacity Gap.}
Define the Pareto endpoints as the $M$ extreme points on the Pareto frontier:
\begin{equation}
    \theta_i^* = \arg\min_\theta L_i(\theta), \quad i = 1, \ldots, M
\end{equation}

The model capacity gap measures  how well $K$ Pareto points approximate these $M$ endpoints. For client $i$, the gap is defined as:
\begin{equation}
    \text{Gap}_i(K) = \min_{k \in [K]} L_i(\theta_k^*) - L_i(\theta_i^*)
\end{equation}

\begin{lemma}[Pareto Coverage Bound]
    \label{lem:pareto_bound}
    Under a clustering assumption where clients can be approximately partitioned into $K$ groups with balanced sizes, the average model capacity gap is bounded by:
    \begin{equation}
        \frac{1}{M}\sum_{i=1}^M \text{Gap}_i(K) \leq \left(1 - \frac{K}{M}\right) \cdot \Delta_{het}
    \end{equation}
\end{lemma}

\begin{proof}
    Assume clients partition into $K$ groups $\mathcal{G}_1, \ldots, \mathcal{G}_K$ with $|\mathcal{G}_k| \approx M/K$. Under the clustering assumption, the K-for-M solution aligns with this partition: each group $\mathcal{G}_k$ has a representative client $r_k$ whose optimal model $\theta_{r_k}^*$ is included in $\Theta^{(K)}_*$.

    For each group $\mathcal{G}_k$:
    \begin{itemize}
        \item \textbf{Representative clients} ($K$ clients): For $i = r_k$, we have $\text{Gap}_{r_k}(K) = 0$ since $\theta_{r_k}^* \in \Theta^{(K)}_*$.
        \item \textbf{Non-representative clients} ($M-K$ clients): For $i \in \mathcal{G}_k \setminus \{r_k\}$:
              \begin{equation}
                  \text{Gap}_i(K) = L_i(\theta_{r_k}^*) - L_i(\theta_i^*) \leq \Delta_{het}
              \end{equation}
    \end{itemize}

    Averaging over all $M$ clients:
    \begin{align}
        \frac{1}{M}\sum_{i=1}^M \text{Gap}_i(K) & = \frac{1}{M} \left[\sum_{k=1}^K 0 + \sum_{k=1}^K \sum_{i \in \mathcal{G}_k \setminus \{r_k\}} \text{Gap}_i(K)\right] \\
                                                & \leq \frac{1}{M} \cdot (M-K) \cdot \Delta_{het}                                                                       \\
                                                & = \left(1 - \frac{K}{M}\right) \cdot \Delta_{het}
    \end{align}

    Note that when $K=M$, every client is a representative, yielding zero gap.
\end{proof}

\subsubsection{Statistical Convergence Analysis}
\label{sec:statistical}

We now analyze the statistical error arising from learning with finite samples.

With finite samples, we optimize empirical losses $\{\hat{L}_i(\theta)\}_{i=1}^M$ instead of population losses $\{L_i(\theta)\}_{i=1}^M$. The empirical K-for-M solution is:
\begin{equation}
    \Theta^{(K)} = \arg\min_{\theta_1, \ldots, \theta_K} \sum_{i=1}^M \min_{k \in [K]} \hat{L}_i(\theta_k)
\end{equation}

\subsubsection{Convergence Bound}

\begin{lemma}[Statistical Convergence]
    \label{lem:statistical}
    For a hypothesis class $\Theta$ with VC dimension $d$, with probability at least $1 - \delta$:
    \begin{align}
         & \frac{1}{M}\sum_{i=1}^M \left|\min_{k} L_i(\theta_k) - \min_{k} L_i(\theta_k^*)\right| \notag \\
         & \qquad \leq \mathcal{O}\left(\sqrt{\frac{Kd + \log(M/\delta)}{n}}\right)
    \end{align}
    where $n$ is the average sample size per client.
\end{lemma}

\begin{proof}
    The K-for-M problem optimizes over the product space $\Theta^K$ with VC dimension $\mathcal{O}(Kd)$. By standard ERM analysis, for each client $i$:
    \begin{align}
         & \min_{k} L_i(\theta_k) - \min_{k} L_i(\theta_k^*) \notag                                                           \\
         & = \underbrace{\left[\min_{k} L_i(\theta_k) - \min_{k} \hat{L}_i(\theta_k)\right]}_{\leq \epsilon(n, Kd)} \notag    \\
         & \quad + \underbrace{\left[\min_{k} \hat{L}_i(\theta_k) - \min_{k} \hat{L}_i(\theta_k^*)\right]}_{\leq 0} \notag    \\
         & \quad + \underbrace{\left[\min_{k} \hat{L}_i(\theta_k^*) - \min_{k} L_i(\theta_k^*)\right]}_{\leq \epsilon(n, Kd)}
    \end{align}
    where the first and third terms are bounded by uniform convergence (Theorem 6.8 in~\cite{shalev2014understanding}), and the second term is non-positive since $\Theta^{(K)}$ minimizes the empirical objective.

    Applying a union bound over $M$ clients and absorbing logarithmic factors yields the stated bound.
\end{proof}

\begin{remark}[Union Bound]
    The union bound (Boole's inequality) states that $P[\bigcup_{i=1}^M E_i] \leq \sum_{i=1}^M P[E_i]$. For $M$ clients each with failure probability $\delta/M$, the total failure probability is at most $\delta$. The exact probability is $1-(1-\delta/M)^M \approx \delta$ for small $\delta/M$.
\end{remark}

\subsubsection{Combining the Bounds}
\label{sec:combining}

The total error for the empirical K-for-M solution decomposes as:
\begin{align}
     & \frac{1}{M}\sum_{i=1}^M \left[\mathbb{E}[L_i(\theta_{k_i})] - L_i(\theta_i^*)\right] \notag                                              \\
     & = \underbrace{\frac{1}{M}\sum_{i=1}^M \left[L_i(\theta_{k_i}^*) - L_i(\theta_i^*)\right]}_{\text{Pareto coverage gap}} \notag            \\
     & \quad + \underbrace{\frac{1}{M}\sum_{i=1}^M \left[\mathbb{E}[L_i(\theta_{k_i})] - L_i(\theta_{k_i}^*)\right]}_{\text{statistical error}}
\end{align}
where $k_i = \arg\min_k L_i(\theta_k)$ for each client $i$.

Combining Lemmas~\ref{lem:pareto_bound} and~\ref{lem:statistical}:
\begin{align}
     & \frac{1}{M}\sum_{i=1}^M \left\{\mathbb{E}\left[\min_{k \in [K]} L_i(\theta_k)\right] - L_i(\theta_i^*)\right\} \notag \\
     & \leq \frac{M-K}{M} \cdot \Delta_{het} + \mathcal{O}\left(\sqrt{\frac{Kd}{n}}\right)
\end{align}

This completes the proof of Theorem~\ref{thm:convergence}.

\begin{remark}[Interpretation of the Bound]
    The bound reveals a fundamental trade-off in the K-for-M framework:
    \begin{itemize}
        \item \textbf{Pareto coverage gap} $(M-K)/M \cdot \Delta_{het}$: Decreases with $K$, vanishes when $K=M$
        \item \textbf{Statistical error} $\mathcal{O}(\sqrt{Kd/n})$: Increases with $K$ due to larger hypothesis class
        \item \textbf{Optimal $K$}: Balances model expressiveness against sample efficiency
        \item \textbf{Asymptotic behavior}: As $n \to \infty$, statistical error vanishes, leaving only the coverage gap
    \end{itemize}
\end{remark}

\begin{table*}[t]
    \centering
    \caption{Training hyperparameters for benchmark and medical datasets. Benchmark datasets use 2000 rounds with CNN/TextCNN backbones, while medical datasets use 1000 rounds with ResNet-18 to prevent overfitting on smaller institutional samples. All experiments employ single local epoch and full client participation for fair comparison across methods.}
    \label{tab:hyperparameters}
    \begin{tabular}{llcccccc}
        \toprule
        \textbf{Type}              & \textbf{Dataset} & \textbf{Model} & \textbf{Rounds} & \textbf{Local Epochs} & \textbf{Batch Size} & \textbf{Learning Rate} & \textbf{Join Ratio} \\
        \midrule
        \multirow{5}{*}{Benchmark} & CIFAR-10         & CNN            & 2000            & 1                     & 50                  & 0.005                  & 1.0                 \\
                                   & CIFAR-100        & CNN            & 2000            & 1                     & 50                  & 0.005                  & 1.0                 \\
                                   & TinyImageNet     & CNN            & 2000            & 1                     & 50                  & 0.0005                 & 1.0                 \\
                                   & AG News          & TextCNN        & 2000            & 1                     & 100                 & 0.005                  & 1.0                 \\
                                   & FEMNIST          & CNN            & 2000            & 1                     & 100                 & 0.005                  & 1.0                 \\
        \hline
        \multirow{2}{*}{Medical}   & Kvasir           & ResNet-18      & 1000            & 1                     & 100                 & 0.002                  & 1.0                 \\
                                   & FedISIC          & ResNet-18      & 1000            & 1                     & 50                  & 0.005                  & 1.0                 \\
        \bottomrule
    \end{tabular}
\end{table*}

\subsection{Theorem~\ref{thm:smooth_error}: Smooth Approximation}
\label{sec:proof_approx}

We prove that $g^{\text{STCH-Set}}(\Theta_K)$ uniformly approximates $g^{\text{TCH-Set}}(\Theta_K)$ by deriving tight upper and lower bounds using standard log-sum-exp approximation properties.

\begin{proof}
    The log-sum-exp function provides well-known smooth approximations for $\max$ and $\min$ operators~\cite{bertsekas2003convex, boyd2004convex}. For any $y_1, \ldots, y_n$ and smoothing parameter $\mu > 0$:
    \begin{align}
                      & \mu \log \sum_{i=1}^n e^{y_i / \mu} - \mu \log{n} \notag                    \\
        \leq{} \,\,\, & \max \{y_1, \ldots, y_n\} \notag                                            \\
        \leq{} \,\,\, & \mu \log \sum_{i=1}^n e^{y_i/\mu}, \label{eq:lse_max_bound}                 \\[4pt]
                      & -\mu \log \sum_{i=1}^n e^{-y_i / \mu} \notag                                \\
        \leq{} \,\,\, & \min \{y_1, \ldots, y_n\} \notag                                            \\
        \leq{} \,\,\, & -\mu \log \sum_{i=1}^n e^{-y_i/\mu} + \mu \log{n}. \label{eq:lse_min_bound}
    \end{align}

    For $g^{\text{TCH-Set}}(\Theta_K) = \max_{i \in [M]} \min_{k \in [K]} L_i(\theta_k)$, we apply~\eqref{eq:lse_min_bound} to the inner minimization and~\eqref{eq:lse_max_bound} to the outer maximization. Define the smooth inner minimum:
    \begin{equation}
        \tilde{m}_i := -\mu \log \left(\sum_{k=1}^K e^{-L_i(\theta_k)/\mu} \right).
    \end{equation}

    By~\eqref{eq:lse_min_bound}, we have $\tilde{m}_i - \mu \log K \leq \min_{k \in [K]} L_i(\theta_k) \leq \tilde{m}_i$. Applying~\eqref{eq:lse_max_bound} to $\max_{i \in [M]} \tilde{m}_i$ and noting that $g^{\text{STCH-Set}}(\Theta_K) = \mu \log (\sum_{i=1}^M e^{\tilde{m}_i/\mu})$:
    \begin{align}
        g^{\text{TCH-Set}}(\Theta_K)
         & = \max_{i \in [M]} \min_{k \in [K]} L_i(\theta_k) \nonumber                                  \\
         & \geq \max_{i \in [M]} (\tilde{m}_i - \mu \log K) \nonumber                                   \\
         & \geq g^{\text{STCH-Set}}(\Theta_K) - \mu \log M - \mu \log K, \label{eq:lower_bound}         \\[4pt]
        g^{\text{TCH-Set}}(\Theta_K)
         & \leq \max_{i \in [M]} \tilde{m}_i \leq g^{\text{STCH-Set}}(\Theta_K). \label{eq:upper_bound}
    \end{align}

    Combining~\eqref{eq:lower_bound} and~\eqref{eq:upper_bound} yields the uniform approximation bound with error $\mathcal{O}(\mu \log M + \mu \log K)$, which vanishes as $\mu \to 0$.
\end{proof}

\subsection{Theorem~\ref{thm:pareto_properties}: Pareto Properties}
\label{sec:proof_pareto}

We prove both parts of Theorem~\ref{thm:pareto_properties}: Pareto optimality and Pareto stationarity.

\subsubsection{Part 1: Pareto Optimality}

\begin{proof}
    The smooth Tchebycheff set scalarization objective:
    \begin{equation}
        g^{\text{STCH-Set}}(\Theta_K) = \mu \log \sum_{i=1}^M \left(\sum_{k=1}^K \exp\left(-\frac{L_i(\theta_k)}{\mu}\right)\right)^{-1}
    \end{equation}

    Let $\Theta^*_K = \{\theta_1^*, \ldots, \theta_K^*\}$ be an optimal solution set. We need to show that each $\theta_k^* \in \Theta^*_K$ is Pareto optimal.

    \textbf{Part 1: Weak Pareto Optimality.}
    Suppose for contradiction that $\Theta^*_K$ is not weakly Pareto optimal. Then there exists another set $\Theta'_K$ such that $L_i(\theta'_{k(i)}) < L_i(\theta^*_{k(i)})$ for all $i \in [M]$, where $k(i) = \arg\min_k L_i(\theta_k)$.

    Since all objectives strictly decrease:
    \begin{align}
        \sum_{k=1}^K \exp\left(-\frac{L_i(\theta'_k)}{\mu}\right) & > \sum_{k=1}^K \exp\left(-\frac{L_i(\theta^*_k)}{\mu}\right) \quad \forall i
    \end{align}

    This implies:
    \begin{align}
        \left(\sum_{k=1}^K \exp\left(-\frac{L_i(\theta'_k)}{\mu}\right)\right)^{-1} & < \left(\sum_{k=1}^K \exp\left(-\frac{L_i(\theta^*_k)}{\mu}\right)\right)^{-1} \quad \forall i
    \end{align}

    Therefore $g^{\text{STCH-Set}}(\Theta'_K) < g^{\text{STCH-Set}}(\Theta^*_K)$, contradicting the optimality of $\Theta^*_K$.

    \textbf{Strong Pareto Optimality.}
    Under either condition (unique optimal set or all positive preferences), we can strengthen the result to Pareto optimality. When the optimal set is unique, any Pareto-dominating solution would yield a strictly better objective value, contradicting uniqueness. When all preferences are positive, the scalarization ensures that improving any subset of objectives without harming others strictly decreases the objective, again contradicting optimality.
\end{proof}

\subsubsection{Part 2: Pareto Stationarity}

We prove that stationary points of STCH-Set are Pareto stationary for the original multi-objective problem.

\begin{proof}
    Consider a point $\hat{\Theta} = \{\hat{\theta}_1, \ldots, \hat{\theta}_K\}$ where gradient descent has converged. The gradient of STCH-Set with respect to model $\hat{\theta}_k$ is:
    \begin{equation}
        \nabla_{\hat{\theta}_k} g^{\text{STCH-Set}} = \sum_{i=1}^M \alpha_i \cdot w_{ik} \cdot \nabla_{\hat{\theta}_k} L_i(\hat{\theta}_k)
    \end{equation}

    where:
    \begin{align}
        w_{ik}   & = \frac{\exp(-L_i(\hat{\theta}_k)/\mu)}{\sum_{j=1}^K \exp(-L_i(\hat{\theta}_j)/\mu)} \\
        \alpha_i & = \frac{S_i^{-1}}{\sum_{j=1}^M S_j^{-1}}
    \end{align}

    At a stationary point where $\nabla_{\hat{\theta}_k} g^{\text{STCH-Set}} = 0$ for all $k$, we have:
    \begin{equation}
        \sum_{i=1}^M \bar{w}_i \nabla_{\hat{\theta}_k} L_i(\hat{\theta}_k) = 0
    \end{equation}

    where $\bar{w}_i = \alpha_i \cdot w_{ik} \geq 0$ and $\sum_i \bar{w}_i = 1$ (forms a convex combination).

    This is precisely the Pareto stationarity condition: the zero vector can be expressed as a convex combination of the individual gradients, meaning no common descent direction exists that improves all objectives simultaneously.
\end{proof}

\section{Experimental Details}
\label{sec:algorithm_suppl}

\begin{figure*}[t]
    \begin{minipage}[t]{0.62\textwidth}
        \centering
        \includegraphics[width=\textwidth]{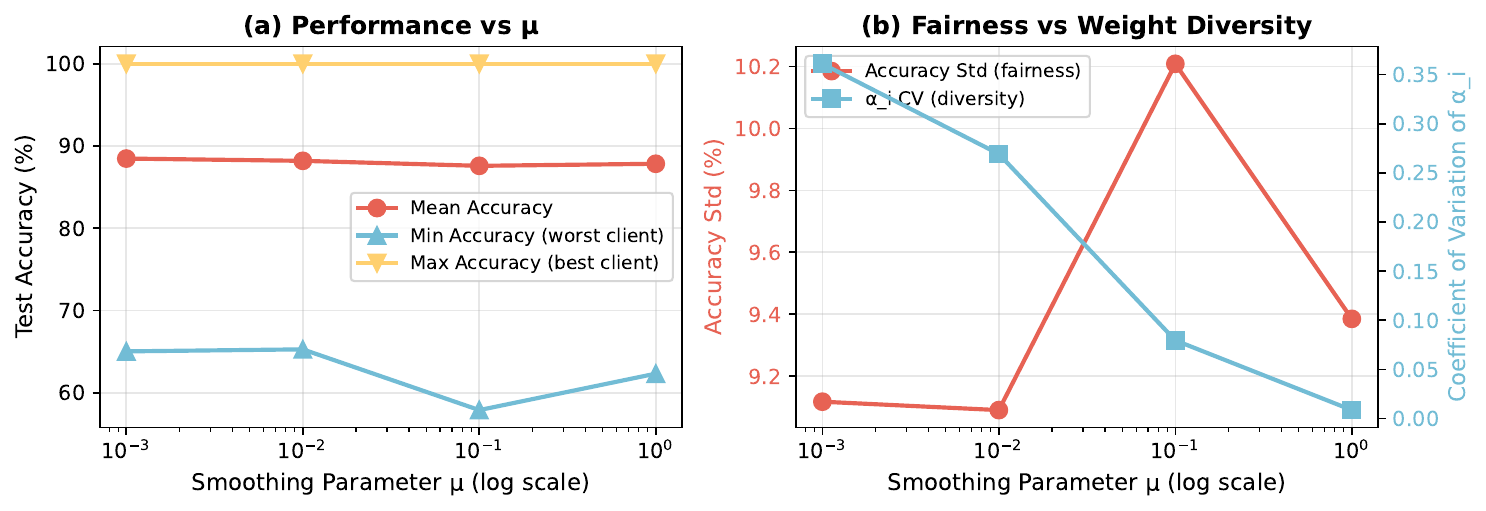}
        \captionof{figure}{\textbf{Fairness and weight diversity analysis.} (a) Performance metrics across different $\mu$ values show that mean accuracy is relatively stable, but worst-case (minimum) accuracy drops significantly at $\mu=0.1$, suggesting a phase transition region. (b) The relationship between fairness (accuracy standard deviation, red) and outer weight diversity (coefficient of variation of $\alpha_i$, blue) reveals that both extreme values ($\mu \to 0$ and $\mu \to \infty$) achieve better fairness than intermediate values, with outer weight diversity decreasing monotonically as $\mu$ increases.}
        \label{fig:mu_fairness_suppl}
    \end{minipage}
    \hfill
    \begin{minipage}[t]{0.32\textwidth}
        \centering
        \includegraphics[width=\textwidth]{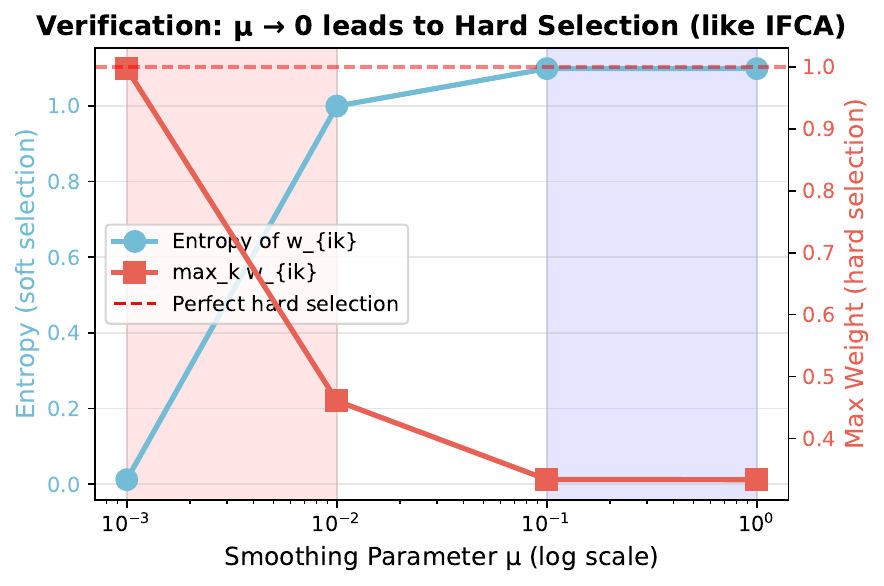}
        \captionof{figure}{\textbf{Theoretical verification: $\mu \to 0$ recovers hard clustering.} Left axis (blue): entropy of inner weights $w_{ik}$ increases with $\mu$, indicating softer model selection. Right axis (red): maximum inner weight decreases with $\mu$, moving away from one-hot assignments. The shaded regions indicate hard selection ($\mu < 0.01$, red) and soft selection ($\mu > 0.1$, blue) regimes.}
        \label{fig:mu_theory_verification_suppl}
    \end{minipage}
    \vspace{-1\baselineskip}
\end{figure*}

\begin{figure*}[t]
    \centering
    \includegraphics[width=0.88\textwidth]{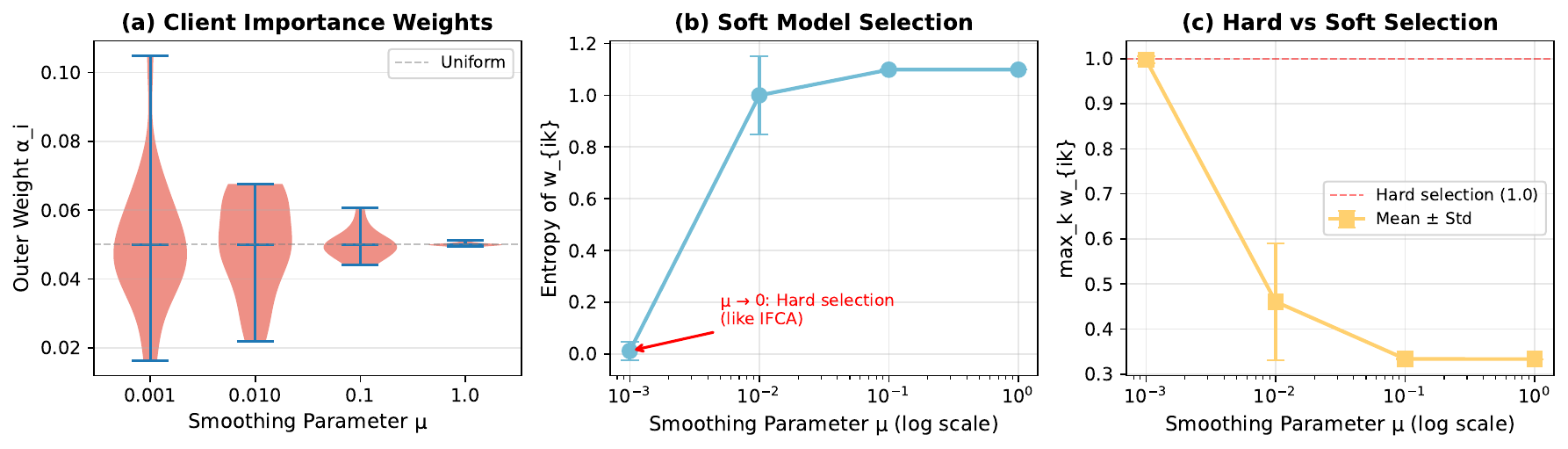}
    \vspace{-0.5\baselineskip}
    \caption{\textbf{Impact of $\mu$ on dual-layer weights.} (a) Outer weights $\alpha_i$ distribution across clients. (b) Inner weight entropy (higher = softer selection). (c) Max inner weight (closer to 1 = harder selection). As $\mu \to 0$, FedFew recovers hard clustering like IFCA.}
    \label{fig:ablation_mu_weights}
    \vspace{-1.0\baselineskip}
\end{figure*}

\subsection{Training Hyperparameters}

Table~\ref{tab:hyperparameters} presents the complete training configurations for all datasets evaluated in our experiments. We categorize the datasets into benchmark and medical imaging domains, each requiring tailored hyperparameter settings due to their distinct characteristics.

\textbf{Benchmark datasets.}
We evaluate on five diverse benchmark datasets (CIFAR-10, CIFAR-100, TinyImageNet, AG News, FEMNIST) spanning vision and text domains.
All benchmark datasets use 2000 communication rounds to ensure convergence across heterogeneous client distributions. For vision datasets, CIFAR-10 and CIFAR-100 use CNN backbones with batch size 50 and learning rate 0.005, balancing training stability with limited samples per client.
TinyImageNet employs the same CNN architecture and batch size, but uses a reduced learning rate of 0.0005 to accommodate its higher resolution (64$\times$64) and larger number of classes (200).
For text classification, AG News uses TextCNN with batch size 100 and learning rate 0.005, leveraging vocabulary diversity to mitigate overfitting. FEMNIST adopts CNN with batch size 100 and learning rate 0.005, benefiting from natural user partitioning that reduces overfitting tendencies.

\textbf{Medical datasets.}
We include two medical imaging datasets (Kvasir, FedISIC) representing real-world healthcare scenarios.
Both use ResNet-18 backbones and 1000 communication rounds, as medical data exhibits faster convergence and higher overfitting risks due to smaller sample sizes per institution.
Kvasir, focusing on gastrointestinal disease classification, uses batch size 100 and a conservative learning rate of 0.002 to handle fine-grained categories while exploiting data augmentation.
FedISIC, dealing with skin lesion classification from small medical centers, adopts batch size 50 and learning rate 0.005 to prevent overfitting on limited training samples.

\textbf{Common settings.}
Across all datasets, we fix local epochs to 1 and join ratio to 1.0 (full client participation) to ensure fair comparison between different personalized FL methods. These settings align with standard practices in federated learning benchmarks.

\textbf{Algorithm-specific hyperparameters.}
For FedFew and IFCA, we use $K=3$ server models across all experiments, which provides a good balance between model expressiveness and optimization complexity as validated in Section~\ref{subsec:ablation_K}. For FedFew specifically, we set the smoothing parameter $\mu=0.01$ for the STCH-Set objective, which enables effective soft model selection while maintaining stable optimization (see Section~\ref{sec:mu_ablation_suppl} for sensitivity analysis).

\subsection{Sensitivity Analysis on Smoothing Parameter}
\label{sec:mu_ablation_suppl}

We investigate the impact of the smoothing parameter $\mu$ on both the dual-layer weight mechanism and overall performance. Figure~\ref{fig:ablation_mu_weights} illustrates how $\mu$ controls the balance between hard and soft model selection. As theory predicts, when $\mu \to 0$, the inner weights $w_{ik}$ approach one-hot assignments (entropy $\approx 0.012$, max weight $\approx 0.997$), recovering IFCA-style hard clustering. Conversely, for large $\mu$ (e.g., $\mu=1.0$), the weights become nearly uniform (entropy $\approx 1.099 \approx \log 3$, max weight $\approx 0.333$), enabling soft model selection. The outer weights $\alpha_i$ exhibit complementary behavior: smaller $\mu$ values lead to more diverse client importance weights (CV = 0.361 at $\mu=0.001$), emphasizing adaptive up-weighting of harder clients, while larger $\mu$ yields nearly uniform weighting (CV = 0.009 at $\mu=1.0$).
On accuracy, we find that performance remains relatively stable across different $\mu$ values.

We provide additional analysis on the impact of the smoothing parameter $\mu$ beyond what is presented in the main paper. Figure~\ref{fig:mu_fairness_suppl} examines the relationship between weight diversity and fairness. Interestingly, both very small and very large $\mu$ values achieve similar fairness (accuracy std $\approx 9.1$--$9.4\%$), while the intermediate region ($\mu=0.1$) exhibits significantly worse fairness (std $\approx 10.2\%$). This suggests a phase transition phenomenon: when $\mu$ is neither small enough for stable hard clustering nor large enough for effective soft selection, the optimization becomes unstable. The outer weight diversity (measured by coefficient of variation of $\alpha_i$) is highest at small $\mu$ (CV = 0.36), indicating strong differentiation of client importance, and decreases monotonically as $\mu$ increases, approaching uniform weighting (CV = 0.009) at $\mu=1.0$.

Figure~\ref{fig:mu_theory_verification_suppl} provides theoretical verification that $\mu \to 0$ recovers hard clustering behavior (similar to IFCA), while large $\mu$ enables soft model selection. The entropy of inner weights $w_{ik}$ decreases from $\approx 1.099$ (uniform distribution over K=3 models, corresponding to $\log 3$) at $\mu=1.0$ to nearly zero at $\mu=0.001$, while the maximum weight increases from $\approx 0.333$ (uniform) to $\approx 0.997$ (one-hot). This validates our theoretical prediction that the smoothing parameter interpolates between soft and hard selection regimes.

\subsection{Communication Efficiency: Alternative Perspective}
\label{sec:rounds_ablation_suppl}

The main paper presents communication-computation trade-offs by plotting convergence against total local updates (Figure~\ref{fig:ablation_rounds}). Here we provide complementary analysis from the communication efficiency perspective.

\bheading{Convergence vs communication rounds.}
Figure~\ref{fig:rounds_comm_suppl} re-plots the same convergence data against communication rounds rather than total updates. This perspective reveals the dramatic communication savings: while all configurations perform identical total computation (2000 local updates), LE=16 achieves convergence in merely 125 communication rounds compared to 2000 rounds for LE=1---a $16\times$ reduction in network overhead. The convergence curves show that configurations with more local epochs not only reduce communication frequency but also exhibit smoother optimization trajectories, with LE=16 demonstrating the steepest and most stable descent in $g^{\text{STCH-Set}}$ values.

\bheading{Accuracy stability across configurations.}
As shown in the main paper (Figure~\ref{fig:rounds_accuracy_suppl}), despite the 16-fold difference in communication costs between LE=1 and LE=16, mean accuracies remain tightly clustered within 87.8--88.3\%, with a maximum deviation of only 0.5 percentage points. This stability validates two key properties of our STCH-Set optimization: (1) robustness to different synchronization frequencies, and (2) insensitivity to the specific (local epochs, communication rounds) decomposition as long as total computation remains constant. The slight accuracy variation (LE=2 achieves 88.3\% while LE=16 achieves 87.8\%) is practically negligible compared to the substantial communication savings, making LE=8 or LE=16 compelling choices for bandwidth-constrained federated deployments.

\begin{figure}[t]
    \centering
    \includegraphics[width=0.44\textwidth]{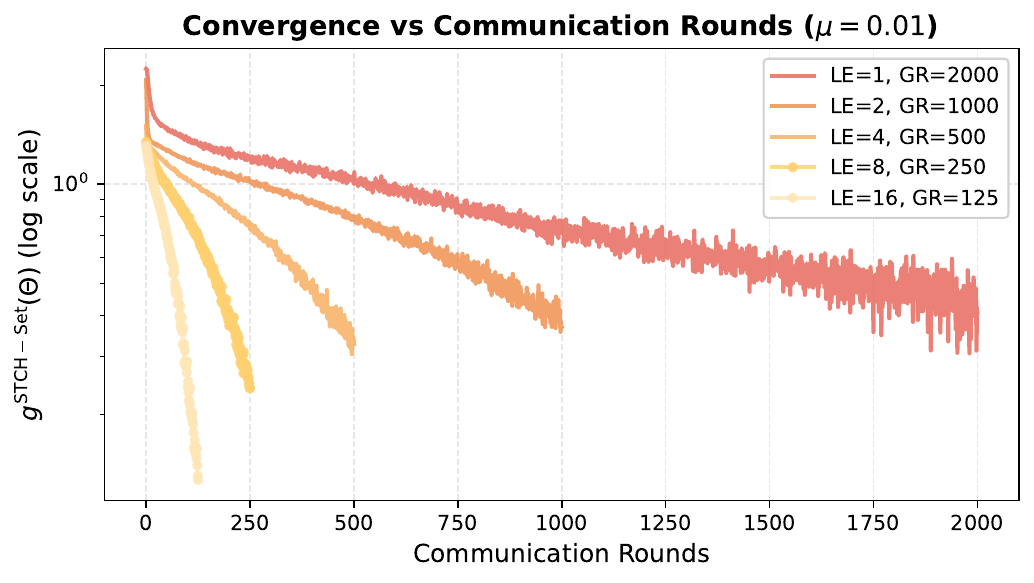}
    \caption{\textbf{Convergence vs communication rounds.} Re-plotting the main paper data against communication rounds instead of total updates reveals the communication efficiency perspective: LE=16 converges in 125 rounds while LE=1 requires 2000 rounds, achieving $16\times$ communication reduction with comparable final $g^{\text{STCH-Set}}$ values.}
    \label{fig:rounds_comm_suppl}
\end{figure}

\subsection{Fairness Analysis}
\label{sec:fairness_suppl}

To evaluate the fairness of personalization across clients, we compute Jain's Fairness Index on per-client test accuracies. A higher index (maximum $J=1$) indicates more equitable performance across clients.

\begin{table}[h]
    \centering
    \caption{Jain's Fairness Index (higher is better, $J=1$ is perfect fairness).}
    \label{tab:jain_suppl}
    \vspace{-0.5\baselineskip}
    \resizebox{0.46\textwidth}{!}{%
        \begin{tabular}{l ccc ccc cc}
            \toprule
                            & \multicolumn{3}{c}{\textbf{CIFAR-10}} & \multicolumn{3}{c}{\textbf{CIFAR-100}} & \multicolumn{2}{c}{\textbf{Medical}}                                                                                      \\
            \cmidrule(lr){2-4} \cmidrule(lr){5-7} \cmidrule(lr){8-9}
            \textbf{Method} & Dir-10                                & Dir-20                                 & Pat-10                               & Dir-10         & Dir-20         & Pat-20         & Kvasir         & FedISIC        \\
            \midrule
            FedAvg          & 0.981                                 & 0.981                                  & 0.982                                & 0.995          & 0.982          & 0.977          & \textbf{0.999} & 0.945          \\
            FedProx         & 0.984                                 & 0.977                                  & 0.976                                & 0.995          & 0.983          & 0.982          & 0.981          & 0.924          \\
            APFL            & 0.992                                 & 0.985                                  & 0.996                                & 0.995          & 0.988          & 0.996          & 0.994          & 0.950          \\
            Ditto           & 0.984                                 & 0.982                                  & 0.997                                & 0.994          & 0.987          & 0.996          & 0.994          & 0.951          \\
            FedRep          & 0.990                                 & 0.985                                  & 0.997                                & 0.995          & 0.991          & 0.996          & 0.995          & 0.938          \\
            IFCA            & 0.992                                 & 0.974                                  & 0.985                                & 0.973          & 0.984          & 0.993          & 0.934          & 0.873          \\
            \textbf{FedFew} & \textbf{0.992}                        & \textbf{0.990}                         & 0.997                                & \textbf{0.996} & \textbf{0.992} & \textbf{0.997} & 0.996          & \textbf{0.958} \\
            \bottomrule
        \end{tabular}
    }
\end{table}

Table~\ref{tab:jain_suppl} shows that FedFew achieves the highest or near-highest Jain's Fairness Index across most settings. Notably, FedFew consistently outperforms IFCA (the other multi-model baseline) in fairness, demonstrating that the soft model selection via STCH-Set provides more equitable personalization than hard clustering.

\subsection{Additional Ablation on $K$: AG News and Kvasir}
\label{sec:vary_k_additional_suppl}

To complement the $K$ ablation on CIFAR-10 in the main paper, we conduct additional experiments on AG News ($K \in \{1, \ldots, 5\}$) and Kvasir ($K \in \{1, \ldots, 5\}$) to investigate whether the optimal $K$ depends on dataset characteristics.

\begin{table}[h]
    \centering
    \caption{Test accuracy (\%) vs.\ number of server models $K$ on AG News and Kvasir datasets.}
    \label{tab:vary_k_agnews_kvasir_suppl}
    \vspace{-0.5\baselineskip}
    \resizebox{0.38\textwidth}{!}{%
        \begin{tabular}{c|ccc|ccc}
            \toprule
                & \multicolumn{3}{c|}{AG News} & \multicolumn{3}{c}{Kvasir}                                                                   \\
            $K$ & Min                          & Mean                       & Max            & Min           & Mean          & Max            \\
            \midrule
            1   & 84.5                         & 96.4                       & 100.0          & 79.5          & 91.6          & 99.5           \\
            2   & 73.8                         & 95.7                       & 100.0          & 82.4          & 92.2          & 99.8           \\
            3   & 83.9                         & 96.0                       & 100.0          & \textbf{83.9} & \textbf{92.9} & 100.0          \\
            4   & 78.0                         & 95.7                       & 100.0          & 82.9          & 92.6          & 100.0          \\
            5   & \textbf{86.3}                & \textbf{96.2}              & \textbf{100.0} & 81.2          & 92.5          & \textbf{100.0} \\
            \bottomrule
        \end{tabular}
    }
\end{table}

Table~\ref{tab:vary_k_agnews_kvasir_suppl} confirms that the optimal $K$ depends on dataset characteristics: Kvasir peaks at $K=3$, while AG News peaks at $K=5$. This aligns with Theorem~\ref{thm:convergence}, which predicts that the optimal $K$ balances the Pareto coverage gap (favoring larger $K$) against statistical error (favoring smaller $K$). Datasets with more heterogeneous client distributions benefit from a larger $K$ to adequately cover the Pareto front.

\subsection{Cost Analysis}
\label{sec:cost_suppl}

We analyze the computational and communication overhead of maintaining $K$ server models compared to single-model approaches.

\begin{itemize}
    \item \textbf{Training}: Clients compute gradients for $K$ models, resulting in a $K\times$ increase in local computation per round. For $K=3$, this represents a modest $3\times$ overhead.
    \item \textbf{Inference}: Each client uses only its best-matching model (determined by $w_{ik}$), so inference cost is identical to single-model methods.
    \item \textbf{Communication}: The overhead scales linearly with $K$ (a constant factor), not with $M$. For $K=3$ and $M=20$, FedFew achieves $>6\times$ reduction in server-side model storage compared to maintaining $M$ personalized models.
    \item \textbf{Server storage}: The server maintains $K$ models instead of $M$, yielding an $M/K$ storage reduction factor.
\end{itemize}

Given that $K \ll M$ and inference cost is unchanged, the trade-off is favorable: a modest increase in training computation yields significant personalization improvements with reduced server-side storage.

\end{document}